%% file: main.tex
\documentclass{article}
\PassOptionsToPackage{numbers, compress}{natbib}

     \usepackage[preprint]{neurips_2025}

\usepackage[utf8]{inputenc} 
\usepackage[T1]{fontenc}    
\usepackage{hyperref}       
\usepackage{url}            
\usepackage{booktabs}       
\usepackage{amsfonts}       
\usepackage{nicefrac}       
\usepackage{microtype}      
\usepackage[dvipsnames]{xcolor}         
\usepackage{colortbl}
\usepackage{dsfont}
\usepackage{stmaryrd}       
\usepackage{graphicx}
\usepackage{xspace}
\usepackage{longtable}
\usepackage{amsthm}
\usepackage{researchpack}
\hypersetup{
    colorlinks=true,
    linkcolor=blue,
    citecolor=blue,
    urlcolor=blue
}

\newcommand{\firstacronym}{Full ABstention in multi-horizon FORecasting\xspace}
\newcommand{\secondacronym}{Partial ABstention in multi-horizon FORecasting\xspace}
\newcommand{\thirdacronym}{INTerval ABstention in multi-horizon FORecasting\xspace}
\newcommand{\firstmethod}{{\tt FAbFor}\xspace}
\newcommand{\secondmethod}{{\tt PAbFor}\xspace}
\newcommand{\thirdmethod}{{\tt IntAbFor}\xspace}
\newcommand{\MLP}{{\tt MLP}\xspace}
\newcommand{\LSTM}{{\tt LSTM}\xspace}

\newcommand{\AdaptiveCF}{{\tt AdaptiveCF}\xspace}
\newcommand{\DPRNN}{{\tt DP-RNN}\xspace}
\newcommand{\MQRNN}{{\tt MQ-RNN}\xspace}
\newcommand{\dummy}{{\tt Accept-cH}\xspace}

\usepackage{math_commands}

\title{Bounded-Abstention Multi-horion \\ Time series Forecasting}
\author{%
  Luca Stradiotti \\
  DTAI lab \& Leuven.AI, \\
  KU Leuven, Belgium \\
  \texttt{luca.stradiotti@kuleuven.be}\\
\And
  Laurens Devos \\
  DTAI lab \& Leuven.AI, \\
  KU Leuven, Belgium \\
  \texttt{laurens.devos@kuleuven.be}\\
  \And
  Anna Monreale \\
  Computer Science Dept, \\
 University of Pisa, Italy \\
  \texttt{anna.monreale@unipi.it}\\
\And
  Jesse Davis\thanks{Joint last authors. Order chosen by alphabetical order.} \\
  DTAI lab \& Leuven.AI, \\
  KU Leuven, Belgium \\
  \texttt{jesse.davis@kuleuven.be} \\
  \And
  Andrea Pugnana\footnotemark[1] \\
  DISI, \\
  University of Trento, Italy\\
  \texttt{andrea.pugnana@unitn.it} \\
}

\begin{document}
\maketitle

\begin{abstract}
Multi-horizon time-series forecasting involves simultaneously making predictions for a consecutive sequence of subsequent time steps. This task arises in many application domains, such as healthcare and finance, where mispredictions can have a high cost and reduce trust. The learning with abstention framework tackles these problems by allowing a model to abstain from offering a prediction when it is at an elevated risk of making a misprediction. Unfortunately, existing abstention strategies are ill-suited for the multi-horizon setting: they target problems where a model offers a single prediction for each instance. Hence, they ignore the structured and correlated nature of the predictions offered by a multi-horizon forecaster. We formalize the problem of learning with abstention for multi-horizon forecasting setting and show that its structured nature admits a richer set of abstention problems. Concretely, we propose three natural notions of how a model could abstain for multi-horizon forecasting.  We theoretically analyze each problem to derive the optimal abstention strategy and propose an algorithm that implements it. Extensive evaluation on 24 datasets shows that our proposed algorithms significantly outperforms existing baselines.

\end{abstract}
\input{texFiles/1_introduction}
\input{texFiles/2_theory}
\input{texFiles/3_implementation}
\input{texFiles/4_experiments}
\input{texFiles/5_relatedWork}
\input{texFiles/6_conclusion}

\section*{Impact Statement}
The deployment of multi-horizon forecasting models in high-stakes applications, such as healthcare and finance, is often hindered by the high cost of mistakes, which can have severe consequences. Enabling multi-horizon forecasters to abstain from making predictions in uncertain scenarios yields significant benefits across these domains.

First, by prioritizing the minimization of selective risk, our framework mitigates the potential consequences of erroneous forecasts, \eg preventing incorrect medical prognoses. By allowing the model to express uncertainty through full, partial, or interval abstention, we enhance the overall safety and reliability of automated decision-making systems. 

Second, our approach facilitates more efficient human-AI collaboration. By identifying specific time steps or intervals where the model is uncertain, we allow human experts to focus their attention and limited resources solely on these challenging steps. This reduces the cognitive load on practitioners, ensuring that human intervention is prioritized for cases where the model is unreliable.
\begin{ack}
LS and JD were supported by the Flemish Government through the “Onderzoeksprogramma Artificiële Intelligentie (AI) Vlaanderen” programme. JD was also supported by the KU Leuven Research Fund (iBOF/21/075). 
LD was supported by Flanders Make, the strategic research centre for the manufacturing industry.
AP and AM acknowledge their work has  been funded by the European Union under Grant Agreement No. 101120763 – TANGO. Views and opinions expressed are however those of the author(s) only and do not necessarily reflect those of the European Union or the European Health and Digital Executive Agency (HaDEA). Neither the European Union nor the granting authority can be held responsible for them.
 \end{ack}

\bibliography{references}

\appendix
\input{texFiles/Appendix}
\end{document}

%% file: texFiles/1_introduction.tex
\section{Introduction}
\begin{figure}[t!]
    \centering
    \includegraphics[width=1\linewidth]{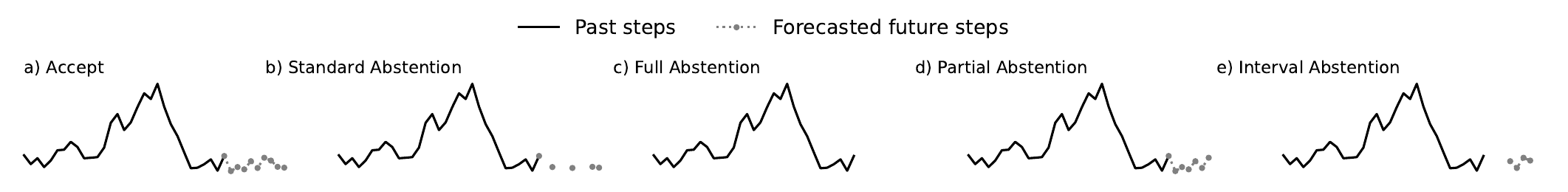}
    \caption{
    Illustration of possible forecasts when using abstention in a multi-horizon setting. (a) The model offers a forecast for the entire horizon, \ie the full forecast is accepted. (b) Standard abstention decides independently for each time step whether to accept or abstain from making a prediction, hence the forecast is not guaranteed to be provided for a contiguous time interval.  (c) Full abstention only allows rejecting  the forecast for the \emph{entire} horizon.  (d) Partial abstention lets the model provide a forecast for part of horizon: the forecast must start from the first time step but it may choose the end time step.  (e) Interval abstention offers the most flexibility: it can offer a forecast for any contiguous interval within the horizon, \ie it can choose any start and end point.
    }
    \label{fig:visual_abstract}
\end{figure}

\textit{Multi-horizon time-series forecasting} involves simultaneously making predictions for a consecutive sequence of subsequent time steps. This problem arises in a wide range of applications, including predicting asset prices~\citep{bao2017deep,selvin2017sotck}, providing weather forecasts~\citep{rasp2020weatherbench}, modeling a patient's prognosis~\citep{alaa2019attentive,lim2018forecasting}, and predicting how an athlete's skill will change over time~\citep{van2025forecasting}.

Many multi-horizon problems occur in high-stakes settings (\eg medicine) where incorrect predictions can have a high cost~\citep{rajkomar2019machine}. Therefore, it is appealing to allow a model to abstain from making a prediction in certain situations~\citep{DBLP:conf/aaai/RuggieriP25}. For example, abstaining may be appropriate when a model is at heightened risk of making a misprediction, \eg because it is highly uncertain about its prediction. In these cases, a decision could be deferred to a human or more information can be collected~\citep{hendrickx2024machine}. 

However, most existing abstention strategies target standard classification and regression problems where a single value (or class) is predicted for each instance. Unfortunately, these methods are ill-suited for the multi-horizon setting because they all treat each prediction separately. That is, for each time step in the horizon, they will independently decide whether to offer a prediction or abstain. This is disadvantageous because, on the one hand, it disregards the clear correlations between consecutive time steps. On the other hand, ignoring the structure of the predictions can lead to fragmented forecasts as illustrated in \Figref{fig:visual_abstract}b. However, in many applications it is necessary to receive a continuous (sub)sequence of predictions for future time steps to identify trends or the onset events. For instance, in weather forecasting, a user relies on a continuous sequence of daily predictions to effectively plan or pack for a trip.

Therefore, our goal is to generalize abstention strategies to account for the inherently richer and more structured nature of the multi-horizon setting. Concretely, we propose and explore three natural multi-horizon abstention settings:

\begin{description}
    \item[Full abstention] is where the forecast is either given for all time steps or none as illustrated in \Figref{fig:visual_abstract}c. This can be viewed as the natural adaptation of abstention for single step predictions to the multi-horizon setting because for each test time series, either a forecast is offered or it is not. For example, this may be appropriate when the test time series is overly noisy due to, \eg a poorly attached sensor that produces motion artifacts.  
    \item[Partial abstention] is where the forecast is offered for a consecutive sequence of time steps starting from the first time step in the prediction horizon as shown in \Figref{fig:visual_abstract}d. For example, a model may produce a more confident forecast about an athletes' performance for the next several seasons, but may highly uncertain about its estimates of later career performance as age-related declines vary substantially among athletes. 
    \item[Interval abstention] generalizes the partial abstention setting by offering a forecast for a consecutive sequence of time steps that may start from \emph{any point} in the prediction horizon as shown in \Figref{fig:visual_abstract}e. For example, in traffic forecasting, it may be difficult to estimate delays immediately following a crash,\footnote{This depends on severity, emergency services' proximity, etc.} but the model will be more certain about its predictions in hour as traffic should return to normal by then.  
\end{description}
We formalize each of these settings using the \textit{bounded-abstention framework}~\citep{franc2023optimal}, which aims to optimize forecasting performance subject to a target coverage constraint (\ie the proportion of examples for which a forecast is offered). For each abstention setting, we theoretically characterize the problem, provide the optimal abstention strategy, and propose an algorithm that approximates the optimal abstention strategy. We extensively validate our framework on $24$ real-world datasets, demonstrating that our approaches significantly outperform several baselines.

%% file: texFiles/2_theory.tex
\section{Formalizing bounded-abstention for multi-horizon time-series forecasting}\label{sec:theory}
We study the forecasting setting where the goal is to learn a forecasting model that generalizes to unseen time series (\eg learn from past patient data to provide forecasts for future patients). We assume time series of the form $y_{1:T}:= (y_1,y_2, ..., y_T)$, where $y_i \in \bbR^d $ and $i$ represents a time step. 

We define a \textit{selective multi-horizon forecaster} $m$ as a pair $(f,g)$, where
\begin{description}
    \item[$f:\calY^{1:T}\to \calY^{T+1:T+H}$] is a \textit{multi-horizon time-series forecaster} that given an input time series $y_{1:T}$, predicts the values of its subsequent $H$ time steps, where $H$ is called the forecast horizon;
    \item[$g:\calY^{1:T}\to \{1, \dots, H\} \times \{0, \dots, H\}$] is a \textit{selection function} that determines the window within the horizon $H$ for which a forecast is offered (accepted steps). For any time step outside this interval, the forecaster abstains (rejected steps). That is, $g\left(y_{1:T}\right) = (s,e)$ with $s \ge 1$ and $e \le H$ denoting the start and the end of the forecast interval. For ease of notation, we omit the explicit dependence on the input time-series and use $(s,e)$ instead of $(s(y_{1:T}), e(y_{1:T}))$. We adopt the convention that $(s=1,e=0)$ denotes that the entire horizon is rejected; apart from this exception we require $s \le e$.
\end{description}  
Formally, the output of $m$ for a time point $T+1 \le i \le T+H$ is defined as:
\begin{equation}\label{eq:general_modelwithreject_equation1}
m_i\left(y_{1:T}\right) = \begin{cases}
    \text{\textregistered} &\text{if $i < s \lor i > e$}\\
    f_i\left(y_{1:T}\right)  & \text{otherwise}.
\end{cases} 
\end{equation}
\noindent That is, the forecasts for the time steps that satisfy the first condition are \emph{rejected} whereas the forecasts for the remaining time steps are \emph{accepted}.

We measure the performance of a \textit{selective multi-horizon forecaster} using the \textit{selective risk}, \ie the expected loss over the accepted steps. Letting $h(g) = e - s+1$ be the number of accepted steps 
and $\phi(g) = \bbE_{Y_{1:T}}\left[h(g)\right]$, we define the selective risk as:
\begin{equation}
    \calR(f, g)=\frac{\bbE_{Y_{1:T+H}}\left[\sum_{t=T+s}^{T+e} \ell\left(y_t, f_t\left(y_{1: T}\right)\right)\right]}
    {\phi(g)}
    \label{eq:sr}
\end{equation}
where $\ell\left(y_t, f_t\left(y_{1: T}\right)\right)$ is the loss at time-step $t$. Intuitively, the numerator represents the \textit{total expected loss} over the accepted steps, while the denominator represents the \textit{expected length} of the forecast interval. To avoid degenerate solutions where the forecaster always abstains, we enforce a \textit{target coverage} $c$, requiring the model to offer forecasts for at least a fraction $c$ of the steps in the horizon on average. Therefore, the objective is to find a selection function $g$ that minimizes the \textit{selective risk} subject to this coverage constraint~\citep{pugnana2024deep}. The general optimization problem can be formalized as:
\begin{equation}
    g^* = \underset{g}{\arg \min} \; \calR(f,g) \quad \text{s.t.} \quad \phi(g) \ge cH.
    \label{eq:problem}
\end{equation}
We now define three different ways to impose structural constraints on how $g$ selects the indices $s$ and $e$, each of which yields a distinct bounded-abstention problem. We theoretically analyze each problem and derive the optimal selection function. In our analysis, we make two assumptions. First, we assume \emph{exchangeability across different time-series}, but we avoid assumptions about dependencies across time steps~\citep{stankeviciute2021conformal}. Second, we assume that $f$ produces all $H$ predictions simultaneously from the input sequence (\ie they are not generated autoregressively). This choice is motivated by the robustness of this approach to error accumulation and conditionally independent predictions~\citep{wen2018multi}.

\subsection{Full abstention for multi horizon forecasting}\label{sec:fullabstention}
We first consider abstention at the \textit{time-series level}:
\begin{definition}[Full Abstention]
        The selective forecaster must accept or reject the entire forecast horizon. Therefore, the selection function in the \textbf{full abstention setting} is defined as $g_{\texttt{FA}}: \mathcal{Y}^{1:T} \rightarrow \{(1,H), (1,0)\}$.
\end{definition}
The goal is to minimize the risk over the accepted time-series while ensuring that, on average, forecasts are provided for at least a fraction $c$ of the time-series.\footnote{This is equivalent to forecast $cH$ time steps on average.} 
Ideally, $g$ should abstain from those time-series with higher expected risk, \ie those for which the forecaster is most uncertain. The core challenge lies in evaluating the total expected risk across the entire horizon, while traditional abstention strategies evaluate it for a single prediction.  We now show that the optimal selection function abstains whenever the sum of conditional risks over the entire horizon exceeds a threshold:
\begin{theorem}
Given a forecaster $f$ and a target coverage $c \in (0,1]$, the optimal selection function $g^*_{\texttt{FA}}$ that is a solution to \Eqref{eq:problem} in full abstention is:
    \begin{equation}
        g^*_{\texttt{FA}}\left(y_{1:T}\right) = \begin{cases} 
        (1,H) & \text { if } \quad \sum_{t=T+1}^{T+H} \rho_t\left(y_{1:T}\right) < \tau^*_c \\ 
        (1,\xi)  & \text { if   } \quad \sum_{t=T+1}^{T+H} \rho_t\left(y_{1:T}\right) = \tau^*_c \\
        (1,0) & \qquad \; \text { otherwise }
        \end{cases}
        \label{eq:optimalg_setting1}
    \end{equation}
where $\xi$ is a random variable such that $\xi = H$ with probability $\kappa$ and $\xi = 0$ with probability $1-\kappa$, and
\begin{equation*}
    \rho_t\left(y_{1:T}\right) =\bbE_{Y_{T+1:T+H}|Y_{1:T}} \left[\ell\left(y_t, f_t(y_{1:T})\right) \mid Y_{1:T} = y_{1:T}\right]
\end{equation*}
    \begin{equation}
        \tau^*_c=\inf \left\{v\mid\bbE_{Y_{1:T}}\left[\bbone\left\{\sum\limits_{t=T+1}^{T+H}\rho_{t}\left(y_{1:T}\right) < v\right\}\right]\geq c\right\}
        \label{eq:tau}
    \end{equation}
\begin{equation*}
\label{eq:kappa}
\kappa = 
\begin{cases} 
    0 & \text{if} \quad  \bbP\left(\sum_{t=T+1}^{T+H}\rho_{t}\left(y_{1:T}\right) = \tau_c^*\right) = 0, \\
    \frac{c - \bbP\left(\sum_{t=T+1}^{T+H}\rho_{t}\left(y_{1:T}\right) < \tau_c^*\right)}{\bbP\left(\sum_{t=T+1}^{T+H}\rho_{t}\left(y_{1:T}\right) = \tau_c^*\right)} & \qquad \qquad \text{otherwise}.
\end{cases}
\end{equation*}
\label{th:setting1}
\end{theorem}
The proof can be found in Appendix~\ref{app:proof_th1}.  

Intuitively, the optimal solution is to abstain for those time-series for which the sum of conditional risks is greater than the $c$-th quantile of the distribution of the sum of conditional risks. Notice that $\xi$ acts as a probabilistic tie-breaker, ensuring coverage constraint is exactly met when the distribution of the sum of the conditional risks is not continuous.

\subsection{Partial abstention for multi horizon forecasting}\label{sec:partialabstention}
Abstaining at the time-series level could be too strict. In some cases, the forecaster may provide accurate predictions for the initial time steps but becomes increasingly uncertain over time. Hence, it can be beneficial to abstain for all predictions after a given time step.
\begin{definition}[Partial Abstention]
    The selective forecaster offers a prediction from $s=1$ to $e$, where the chosen $e$ differs for each considered time series. Therefore, the selection function in the  \textbf{partial abstention setting}
    is defined as $g_{\texttt{PA}}\left(y_{1:T}\right) : \calY^{1:T} \rightarrow \{(1,0), \dots, (1,H)\}$.
\end{definition}

The goal is to minimize the selective risk while ensuring that forecasts are provided for at least a fraction $c$ of the horizon when averaged over all time series. Hence, for each time series, the selective forecaster must consider the trade off between the increase in the conditional risk associated with a longer prediction with the need to meet the global coverage requirement.
We formalize this intuition by showing that the optimal strategy arises from a Lagrangian relaxation, where the selection function balances the cumulative conditional risk against a linear reward for the forecast length.
\begin{theorem}
\label{th:setting2}
Assume the joint vector $\left(\rho_{T+1}(y_{1:T}), \dots, \rho_{T+H}(y_{1:T})\right)$ is absolutely continuous. Given a forecaster $f$ and a target coverage $c \in (0,1]$, the optimal selection function $g^*_{\texttt{PA}}$ in partial abstention returns a pair of indices $\left(s^*, e^*\right)$, such that $s^* = 1$ and:
\begin{equation}
e^* = \underset{e \in \{0, \dots, H\}}{\arg \min} \left( \sum_{t=T+1}^{T+e}\rho_t\left(y_{1:T}\right) - \gamma^*e \right)
\label{eq:optimalg_setting2}
\end{equation}
where $\gamma^* = \lambda^* + \eta^*$, $\lambda^* = \calR\left(f, g^*_{\texttt{PA}}\right)$ is the optimal risk, and $\eta^*$ is the Lagrangian multiplier which is chosen such that the pair $(e^*, \eta^*)$ satisfies the Karush–Kuhn–Tucker (KKT) conditions:
\begin{equation*}
\begin{gathered}   
\eta^* \ge 0, \quad \qquad
\bbE_{Y_{1:T}}\left[e^*\right] \geq cH, \qquad
\eta^*\left(\bbE_{Y_{1:T}}\left[e^*\right] - cH\right) = 0.
\end{gathered}
\end{equation*}
\end{theorem}
The proof can be found in Appendix~\ref{app:proof_th2}. 

Intuitively, the parameter $\gamma^*$ represents the reward of forecasting an additional step. For each test time series, the selective forecaster will provide a prediction for an additional time step only if the associated marginal increase in the cumulative conditional risk is lower than $\gamma^*$.

\subsection{Interval abstention for multi horizon forecasting}\label{sec:intervalabstention}
Finally, we generalize the partial abstention setting by not forcing the provided forecast to begin with first time step.
\begin{definition}[Interval Abstention]
    The selective forecaster selects an interval from $s$ to $e$ within the horizon and offers a prediction for all time steps in that interval. Therefore, the selection function in the \textbf{interval abstention setting} is defined as $g_{\texttt{IA}}\left(y_{1:T}\right) : \calY^{1:T} \rightarrow \{1, \dots, H\} \times \{0, \dots, H\}$. Note that the end time step $e$ must be after the start time step $s$ with the exception that (1,0) is possible (\ie the entire forecast is rejected).
\end{definition}
The objective is to identify the interval that minimizes the selective risk, while ensuring that, on average, the model provides predictions for intervals of length at least $cH$.
We now show that the optimal $g^*_{\texttt{IA}}$ is the minimizer of a Lagrangian trade-off between the cumulative conditional risk and a linear penalty on the interval length.
\begin{theorem}
\label{th:setting3}
Assume the joint vector of conditional risks $(\rho_{1}(y_{1:T}), \dots, \rho_{H}(y_{1:T}))$ is absolutely continuous. Given a forecaster $f$ and a target coverage $c \in (0,1]$, the optimal selection function $g^*_{\texttt{IA}}\xspace$ in interval abstention returns a pair of indices $(s^*, e^*)$ defined as:
\begin{equation}
    (s^*, e^*) = \underset{s \ge 1, e \le H }{\arg \min} \left( \sum_{t=s}^{e} \rho_t(y_{1:T}) - \gamma^*\cdot (e - s + 1)\right)
\label{eq:optimalg_setting3}
\end{equation}
where $\gamma^* = \lambda^*+\eta^*$, and $\lambda^* = \calR(f, g^*_{\texttt{IA}})$ is the optimal risk. The Lagrange multiplier $\eta^*$ is chosen such that the pair $(e^*, \eta^*)$ satisfies the KKT conditions: 
\begin{equation*}
\begin{gathered}   
\eta^* \ge 0, \quad \qquad \bbE_{Y_{1:T}}\left[e^*\right] \geq cH, \qquad
\eta^*\left(\bbE_{Y_{1:T}}\left[e^*\right] - cH\right) = 0.
\end{gathered}
\end{equation*}
\end{theorem}
The proof can be found in Appendix~\ref{app:proof_th3}. 

Again, the parameter $\gamma^*$ captures the expected reward for predicting another step. However, for each test time series, the optimal selective forecaster now will choose the start and end step of the contiguous interval whose cumulative conditional risk is smallest after subtracting the reward $\gamma^*$ for each accepted step.

%% file: texFiles/3_implementation.tex
\section{Learning a Selective Forecaster}\label{sec:implementation}
We now describe how to learn a selective forecaster from data. This requires picking an estimator of the conditional risk, learning a forecasting model, and learning the selection function. The risk estimator and forecaster do not depend on the considered abstention settings, whereas the selection function does.

For the risk estimator, we use the Mean Squared Error (MSE) over the accepted steps, which is a natural choice in multi-horizon forecasting~\citep{zhou2021informer,challu2023nhits}. Thus, the conditional risk coincides with the conditional variance.\footnote{Other risk choices require estimating a different quantity.} To estimate this quantity, we assume access to a dataset $\calD_{\text{train}} = \left\{\left(y^i_{1:T}, y^i_{T+1:T+H}\right), i = 1, \dots, n\right\}$, consisting of $n$ exchangeable time series with observed past and future values. We construct an estimator $\hat{\rho}_t(y_{1:T})$ for each time step $t$ by solving a regression task of the output variable 
$$\hat{\rho}_t\left(y_{1:T}\right)= \hat{\sigma}^2_t(y_{1:T})= \left(y_t - \hat{f}_t\left(y_{1:T}\right)\right)^2$$
where $\hat{\sigma}^2_t$ represents the conditional variance at time $t$. This quantity can be estimated by using a joint training approach, where a single neural network backbone is equipped with two output heads: one head outputs the $H$ future steps $\hat{f}_t$ with $t \in \{T+1, \dots, T+H\}$, and the other outputs the $H$ future step conditional variances $\hat{\sigma}^2_t$ associated with the prediction $\hat{f}_t$. This model can be effectively trained using the $\beta$-NLL loss, which emphasizes well-calibrated variance estimates~\citep{seitzer2022pitfalls}:
\begin{align}
     \;\! L_{{\beta}-{NLL}}\left(y, \hat{f}(y_{1:T}), \hat\sigma^2(y_{1:T})\right)\! = 
     \!\!\!\sum_{t=T+1}^{T+H}\!\!\! \operatorname{s}\!\left(\hat{\sigma}_t^{2\beta}(y_{1:T})\right)\!\!
      \left( \frac{\log\left(\hat{\sigma}_t^2(y_{1:T})\right)}{2}\!+\!\frac{\left(y_t\!-\!\hat{f}_t(y_{1:T})\right)\!^2}{2\hat{\sigma}_t^2(y_{1:T})} \right) 
      \label{eq:loss} 
\end{align}
where $\operatorname{s}(\cdot)$ denotes the stop gradient operation and $\beta$ controls the strength of the weighting.

All three abstention strategies rely on these variance estimates $\hat{\sigma}^2_t$. Then, they enforce the coverage constraint using a separate calibration set $\calD_{\text{calib}} = \{y^{n+i}_{1:T}\}_{i = 1}^m$ of $m$ exchangeable time series.\footnote{Since the thresholds and parameters depend only on the input $y_{1:T}$, future values are not strictly required for calibration.}

\paragraph{Full Abstention} \firstmethod (\firstacronym) approximates the optimal selection function derived in~\Thref{th:setting1}. We enforce the coverage constraint by computing $\sum_{t=T+1}^{T+H}\hat{\sigma}^2_t(y_{1:T})$ for all series in $\calD_{calib}$ and setting $\hat{\tau}_c$ as the empirical $c$-quantile of these scores (see~\Eqref{eq:tau}). Thus the learned selection function accepts (in expected value) a fraction $c$ of the forecasts, thereby satisfying the coverage constraint, and can be written as:
\begin{equation*}
    \hat{g}_{\texttt{FA}}\left(y_{1:T}\right) = \begin{cases} 
    (1,H) & \text { if } \quad \sum_{t=T+1}^{T+H} \hat{\sigma}^2_t\left(y_{1:T}\right) < \hat{\tau}_c \\ 
    (1,\hat{\xi})  & \text { if   } \quad \sum_{t=T+1}^{T+H} \hat{\sigma}^2_t\left(y_{1:T}\right) = \hat{\tau}_c \\
    (1,0) & \qquad \; \text { otherwise }
\end{cases}
\end{equation*}
where $\hat{\xi}$ is a random variable such that $\hat{\xi} = H$ with probability $\hat{\kappa}$ and $\hat{\xi} = 0$ with probability $1-\hat{\kappa}$, and
\begin{equation*}
\hat{\kappa} = 
\begin{cases}
    0 & \text{if } \mathbb{P}\left(\sum_{t=T+1}^{T+H}\hat{\sigma}^2_{t}(y_{1:T}) = \hat{\tau}_c\right) = 0, \\
    \frac{c - \mathbb{P}\left(\sum_{t=T+1}^{T+H}\hat{\sigma}^2_{t}(y_{1:T}) < \hat{\tau}_c\right)}{\mathbb{P}\left(\sum_{t=T+1}^{T+H}\hat{\sigma}^2_{t}(y_{1:T}) = \hat{\tau}_c\right)} & \text{otherwise}.
\end{cases}
\end{equation*}

\paragraph{Partial abstention}\label{sec:practical_partial}  \secondmethod (\secondacronym) approximates the optimal selection function derived in~\Thref{th:setting2}. To efficiently estimate $\gamma^*$ that satisfies the coverage constraint, we employ a binary search algorithm. This is possible because the coverage is monotonic with respect to $\gamma$ (see~\Lemref{lem:monotonicity_gamma} in~\Appref{app:proof_th2}). However, since the empirical joint distribution of conditional variances may not be continuous, there may not exist a $\gamma$ that satisfies the coverage constraint exactly. Therefore, we find $\hat{\gamma_{\ell}}$ and $\hat{\gamma_r}$ that provide the tightest lower and upper bounds on the target coverage on $\calD_{calib}$:
\begin{equation}
\begin{gathered}
    \hat{\gamma_{\ell}} = \operatorname{sup} \left\{\gamma \geq 0 : \bbE_{\calD_{calib}} \left[\hat{g_{\gamma}}\right] \leq cH\right\} \qquad
    \hat{\gamma_r} = \operatorname{inf} \left\{\gamma \geq 0 : \bbE_{\calD_{calib}} \left[\hat{g_{\gamma}}\right] \geq cH\right\}
\end{gathered}
\label{eq:gammas}
\end{equation}
Then, we enforce the coverage constraint exactly by defining $\hat{\gamma}$ as a random variable taking the value $\hat{\gamma}_{\ell}$ with probability $p$ and $\hat{\gamma}_{r}$ with probability $(1-p)$, where
\begin{equation}
    p = \frac{cH - \hat{\phi}(\hat{g}_{\hat{\gamma_r}})}{ \hat{\phi}(\hat{g}_{\hat{\gamma_{\ell}})} -  \hat{\phi}(\hat{g}_{\hat{\gamma_r}})}.
    \label{eq:p}
\end{equation}
Here $\hat{\phi}(\hat{g}_{\hat{\gamma_{\ell}}})$ and $\hat{\phi}(\hat{g}_{\hat{\gamma_r}})$ denote the empirical coverages on $\calD_{calib}$, when using $\hat{\gamma_{\ell}}$ and $\hat{\gamma_r}$ respectively. The resulting selection function $\hat{g}_{\texttt{PA}}$ is obtained as:
\[
    \hat{g}_{\texttt{PA}}\left(y_{1:T}\right) = \left(1, \underset{e \in \{0, \dots, H\}}{\arg \min} \left( \sum_{t=T+1}^{T+e}\hat{\sigma}^2_t\left(y_{1:T}\right) - \hat{\gamma}e \right)\right)
\]
The full algorithm can be found in Appendix~\ref{app:algorithm_pabfor}.

\paragraph{Interval abstention}\label{sec:practical_interval} Finally, \thirdmethod (\thirdacronym) approximates the optimal selection function derived in \Thref{th:setting3}. Similar to \secondmethod, this method relies on the variance estimates $\hat{\sigma}^2_t$, uses a binary search to identify the tightest bounds $\hat{\gamma}_{\ell}$ and $\hat{\gamma}_r$ around the target coverage, and obtains $\hat{\gamma}$ as a random variable induced by these bounds (see \Eqref{eq:gammas} and \Eqref{eq:p}). The key difference lies in the local optimization step required for each time-series. For a fixed $\hat{\gamma}$, instead of simply truncating the horizon, \thirdmethod estimates the optimal interval $(s^*, e^*)$ by solving the minimization problem in \Eqref{eq:optimalg_setting3}. 
We solve this optimization problem efficiently by decomposing it into two steps. First, for every possible forecast length $h \in \{0,\dots, H\}$, we identify the start step $\hat{s}_h$ of the subsequence of length $h$ with the lowest cumulative variance:
$$ 
\hat{s}_{h}(y_{1:T}) = \underset{s \in \{1, \dots, H-h+1\}}{\arg \min} \sum_{t=T+s}^{T+s+h-1} \hat{\sigma}_t^2(y_{1:T}),
$$
Second, among these candidates, we select the interval length $\hat{h}$ that minimizes the Lagrangian cost: 
$$
\hat{h}\left(y_{1:T}\right)=\!\!\underset{h \in \{0, \dots, H\}}{\arg\min}\!\!\left(\sum\limits_{t = T+\hat{s}_h(y_{1:T})}^{T+\hat{s}_h(y_{1:T}) + h-1}\hat{\sigma}^2_t\left(y_{1:T}\right) - \hat{\gamma} h\right).
$$
 
Finally, the selection function $\hat{g}_{\texttt{IA}}$ is obtained as:
\begin{equation*}
    \hat{g}_{\texttt{IA}}\left(y_{1:T}\right) = \left(\hat{s}_{\hat{h}}\left(y_{1:T}\right), \hat{s}_{\hat{h}}\left(y_{1:T}\right) + \hat{h}(y_{1:T}) - 1\right)
\end{equation*}
The full algorithm can be found in Appendix~\ref{app:algorithm_intabfor}.

%% file: texFiles/4_experiments.tex
\section{Experimental Evaluation} 
Empirically, we address the following research questions:
\begin{itemize}
   \item[\textbf{Q1}.] In full abstention, does \firstmethod achieve lower selective risk than its competitors?
   \item[\textbf{Q2}.] Does giving the selection function more fine-grained control over the chosen forecasting interval yield a lower selective risk?
    \item[\textbf{Q3}.] Is the coverage constraint satisfied at test time?
\end{itemize}

\subsection{Experimental Setup}\label{sec:experiments}
\textbf{Baselines.} 
We compare \firstmethod against three natural baselines. \underline{\AdaptiveCF} rejects the time series based on the sum of the widths of the conformal prediction intervals over the $H$-step horizon. We use the adaptive approach, as the standard conformal method yields equal-width intervals for all time-series~\citep{stankeviciute2021conformal}. \underline{\MQRNN} trains the selection function using a multi-quantile loss and rejects time series based on the sum of the predicted quantile intervals' width across the $H$-step horizon~\citep{wen2018multi}. \underline{\DPRNN} employs Monte Carlo dropout during training and inference. The selection function is trained to reject time series with high predictive uncertainty, measured as the sum of the standard deviations across the $H$-step horizon~\citep{gal2016dropout}. 

\textbf{Data.} We evaluate the effectiveness of our strategies on the $24$ datasets described in~\Tabref{tab:datasets} (\Appref{app:data}). We include datasets with diverse characteristics, \eg the number of time-series ranges from $380$ to $38400$, and the forecast horizon ranges between $6$ and $50$. All datasets are publicly available.~\Appref{app:data} provides additional details on the datasets and their preprocessing. 

\textbf{Evaluation metrics.} As standard in the bounded-abstention framework, we consider two key quantities: the \emph{empirical selective risk} and the \textit{empirical coverage}~\citep{geifman2017selective}. Specifically, we compute the selective risk as the average Mean Squared Error per accepted time step (see~\Eqref{eq:sr}). We assess whether the empirical coverage is within a user-defined tolerance $\epsilon$ of the target coverage~\citep{pugnana2024deep}:
\begin{equation}
ConSat(\epsilon) = \bbone_{\frac{\hat{\phi}(\hat{g})}{H} \ge c - \epsilon}\left(\epsilon\right)
\end{equation}
where $\hat{\phi}\left(\hat{g}\right)$ is the empirical coverage  on the test set.

\textbf{Setup.} For each dataset, we employ the following procedure. $(i)$ We randomly split the data into training, calibration, and test set $\left(60\%/20\%/20\%\right)$. $(ii)$ We train the forecaster $\hat{f}$, the conditional variance estimator $\hat{\sigma}^2\left(y_{1:T}\right)$ and the baselines on the training set. $(iii)$ We calibrate the rejection parameters on the calibration set for six target coverage levels $c \in \{0.7, 0.75, 0.8, 0.85, 0.9, 0.95\}$. Specifically, we determine $\hat{\tau}_c$ (for \firstmethod) or $\hat{\gamma}_{\ell}, \hat{\gamma}_{r}, p$ (for \secondmethod and \thirdmethod). For all baselines, we set the rejection threshold as the $c$-th quantile of their score distribution, ensuring that the $(1 - c)\%$ of time series with the largest uncertainty are rejected. $(iv)$ For each coverage level we determine the selection function $\hat{g}$ and $(v)$ compute the empirical selective risk and the empirical coverage on the test set. We repeat all steps using ten random seeds to split the data and initialize the network weights, and report the average results.

\textbf{Model Architectures and Hyperparameters.} The forecaster $f$ relies on a common backbone across all methods, consisting of a single-layer Vanilla \LSTM ~\citep{hochreiter1997long} with 20 hidden units. {\tt LSTMs}\xspace are a natural choice for time-series forecasting due to their ability to capture temporal correlations and long-range dynamics across time steps~\citep{taieb2016bias,stankeviciute2021conformal}. For our proposed strategies (\firstmethod, \secondmethod, and \thirdmethod), this backbone branches into two separate heads -- one for prediction and one for variance estimation -- where each head consists of an MLP with a single hidden layer of $40$ neurons and ReLU activation. These are trained end-to-end using the loss in Equation~\ref{eq:loss} with the variance weight $\beta$ set to $0.5$. Similarly, for the baselines, the \LSTM backbone feeds into a \MLP with a single hidden layer of $40$ neurons and ReLU activation. The baselines uses this architecture but employ different loss functions (\eg quantile loss) to estimate the uncertainty (\eg width of the prediction intervals). Moreover, for \AdaptiveCF, we set the sensitivity parameter $\beta$ to $1$; for \MQRNN, we set the quantile loss parameter $q$ to $0.05$; and for \DPRNN, we apply a dropout rate of $0.1$ and compute the standard deviation over $100$ Monte Carlo samples. All networks are trained for $500$ epochs using the Adam optimizer with a learning rate of $0.001$. We used Python language and an AMD EPYC 9334 32-Core Processor machine with 256 GB RAM and two NVIDIA L40S GPUs to run all our experiments. Our evaluation pipeline requires approximately three days to complete.

\subsection{Experimental Results}\label{sec:results}

\textbf{(Q1) \firstmethod significantly outperforms the competitors.}
\begin{figure*}[!t]
  \centering
  \includegraphics[width=1\textwidth]{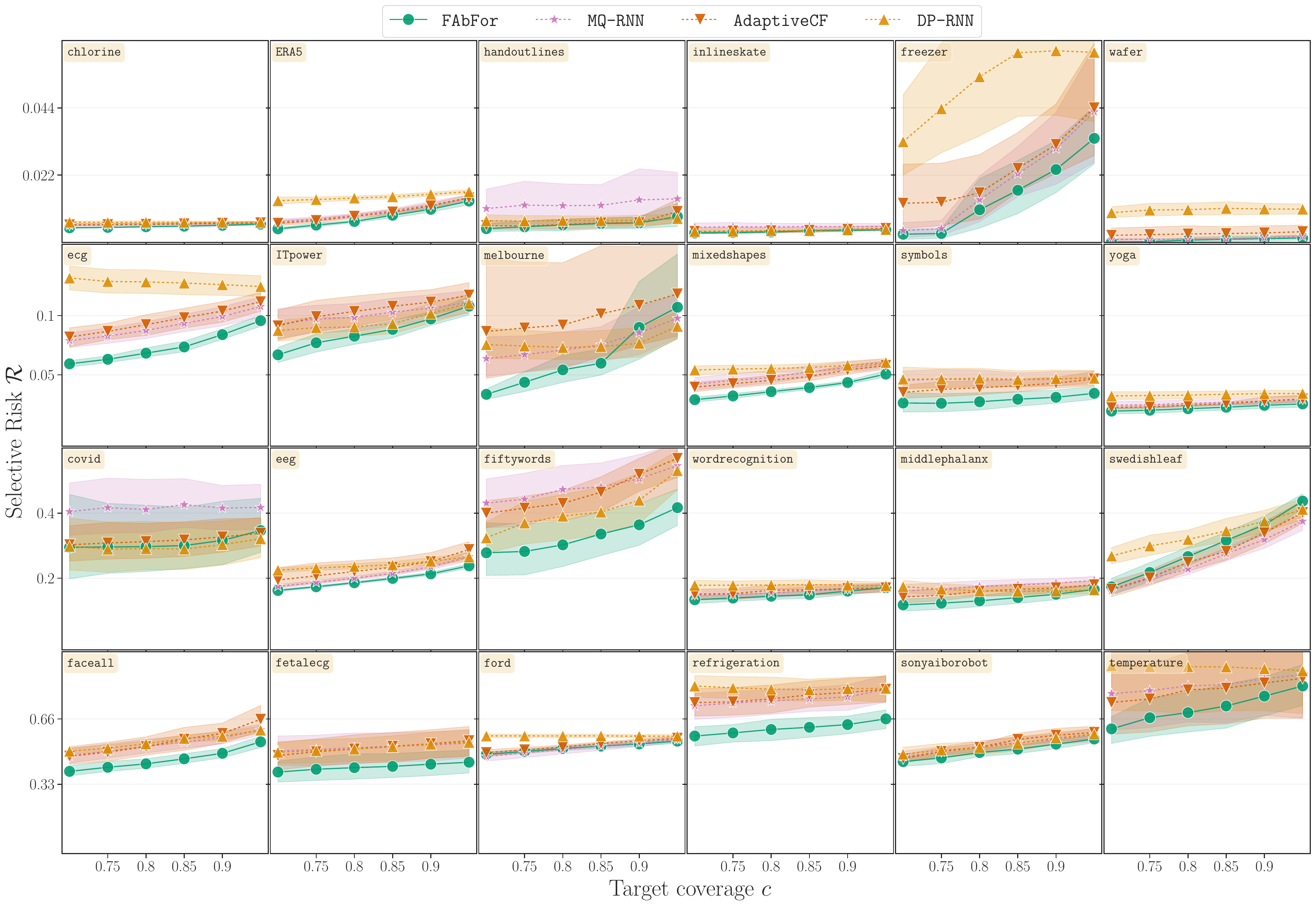}
  \caption{Each dataset's average selective risk $\hat{\calR}$ for all methods across six target coverages in the full abstention setting. For all the coverage level considered, \firstmethod consistently achieves a lower selective risk than the baselines in most of the datasets.}
  \label{fig:Q1}
\end{figure*}
Figure~\ref{fig:Q1} shows each dataset's average empirical selective risk $\hat{\calR}$ across six target coverage levels for all considered methods in the full abstention setting. As expected, the selective risk decreases for all methods as more time-series are rejected. Overall, \firstmethod consistently outperforms all baselines on $22$ out of $24$ datasets, reducing the selective risk by an average of $14\%$ vs \AdaptiveCF and \MQRNN, and $19\%$ vs \DPRNN. Moreover, \firstmethod achieves lower selective risk in around $80\%$ of the experiments against the two runner-ups, \AdaptiveCF and \MQRNN.

\begin{figure*}[!t]
  \centering
  \includegraphics[width=1\textwidth]{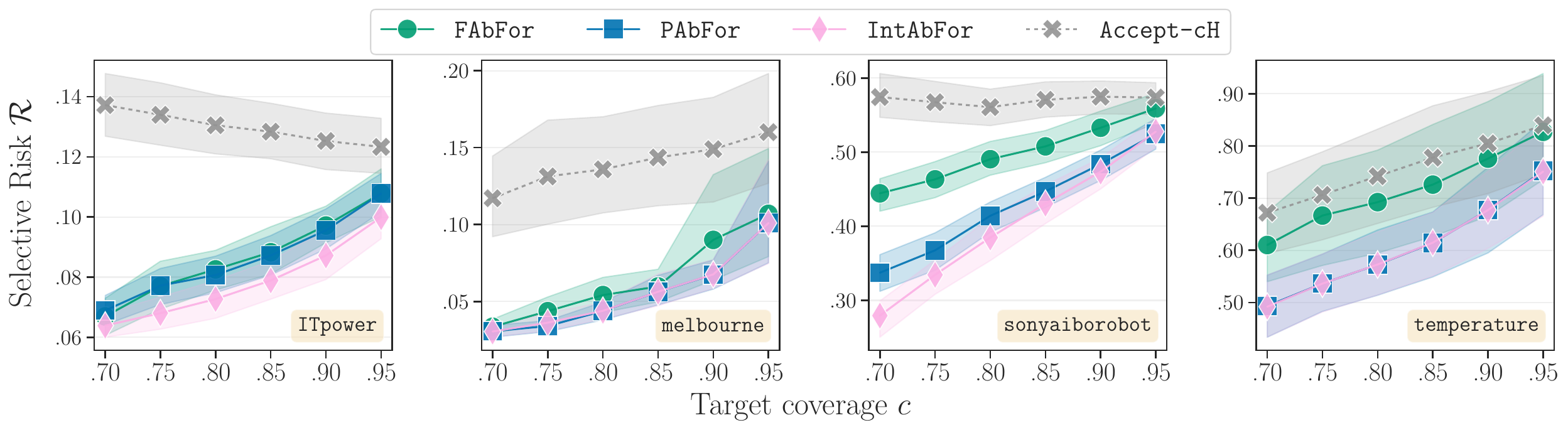}
  \caption{
  Average selective risk per step for \firstmethod, \secondmethod, \thirdmethod, and \dummy across six target coverages on four representative datasets. \thirdmethod achieves often the lowest risk by leveraging the highest flexibility, followed by, in order, \secondmethod and \firstmethod. \dummy, which forecasts all time-series until step $cH$, performs worse.
  }
  \label{fig:Q2}
\end{figure*}

\textbf{(Q2) Finer-grained control in the abstention strategy yields lower selective risk.}
To assess if having more fine-grained control over the selected interval yields a lower selective risk, we compare our three proposed methods to a naive baseline \underline{\dummy} that predicts the first $cH$ steps in the horizon.\footnote{To exactly enforce the coverage constraint, it always accepts the first $\lfloor c H\rfloor$ steps and accepts the $\left(\lfloor c H\rfloor +1\right)$-th step with probability $p = cH - \lfloor c H\rfloor$.} We omit the other \textbf{Q1} baselines because they perform worse than \firstmethod and their selection functions only permit full abstention. 

\Figref{fig:Q2} reports the average selective risk for the three proposed strategies and \dummy across six target coverage levels on four representative datasets (see Appendix~\ref{sec:Q2app} for full results). All our approaches significantly outperform \dummy, confirming that an effective abstention strategy requires assessing the cumulative conditional risk.

Next, we rank the methods from best (rank 1) to worst (rank 4) for each experiment and report the average ranks in~\cref{tab:Q2}. Across all coverages, we observe a performance hierarchy consistent with the level of control of the selection function: \thirdmethod consistently achieves the lowest (best) rank, followed by \secondmethod and \firstmethod. Quantitatively, averaging across all datasets, \secondmethod reduces the selective risk by $7\%$ compared to \firstmethod, while \thirdmethod achieves a further $2\%$ reduction compared to \secondmethod. The similar performance of \secondmethod and \thirdmethod arises because uncertainty typically increases over the time horizon. Hence, \thirdmethod's interval typically starts at $s\!=\!1$, resulting in the same forecast interval as \secondmethod. However, in some datasets the added flexiblity is very important. For example, on \texttt{ITpower} and \texttt{sonyaiborobot}, \thirdmethod achieves notably lower risk than \secondmethod: on this dataset, \thirdmethod's selection function sets $s > 1$ for nearly $40\%$ of the test time-series.

\begin{table}[t]
\centering
\caption{Average rank ($\pm$ std.) for each method across all datasets for six target coverage levels. 
}
\begin{tabular}{c|ccc|c}
\toprule
{} & \multicolumn{4}{c}{\textbf{Ranks} (avg. $\pm$ std.)}\\
$c$ & \firstmethod & \secondmethod & \thirdmethod & \dummy \\
 \midrule
.70 & 2.77 $\pm$ 0.84 & 1.72 $\pm$ 0.67 & \textbf{1.54 $\pm$ 0.70} & 3.85 $\pm$ 0.40 \\
.75 & 2.78 $\pm$ 0.84 & 1.68 $\pm$ 0.65 & \textbf{1.59 $\pm$ 0.71} & 3.85 $\pm$ 0.41 \\
.80 & 2.78 $\pm$ 0.81 & 1.65 $\pm$ 0.67 & \textbf{1.59 $\pm$ 0.72} & 3.85 $\pm$ 0.40 \\
.85 & 2.79 $\pm$ 0.76 & 1.68 $\pm$ 0.67 & \textbf{1.57 $\pm$ 0.73} & 3.85 $\pm$ 0.44 \\
.90 & 2.69 $\pm$ 0.85 & 1.72 $\pm$ 0.67 & \textbf{1.60 $\pm$ 0.75} & 3.85 $\pm$ 0.45 \\
.95 & 2.58 $\pm$ 0.86 & 1.73 $\pm$ 0.73 & \textbf{1.64 $\pm$ 0.80} & 3.86 $\pm$ 0.45 \\
\bottomrule
\end{tabular}
\label{tab:Q2}
\end{table}
\textbf{(Q3) Finer-grained control in the abstention strategy enables precise coverage satisfaction.}
\begin{figure}[!t]
  \centering
  \includegraphics[width=.45\textwidth]{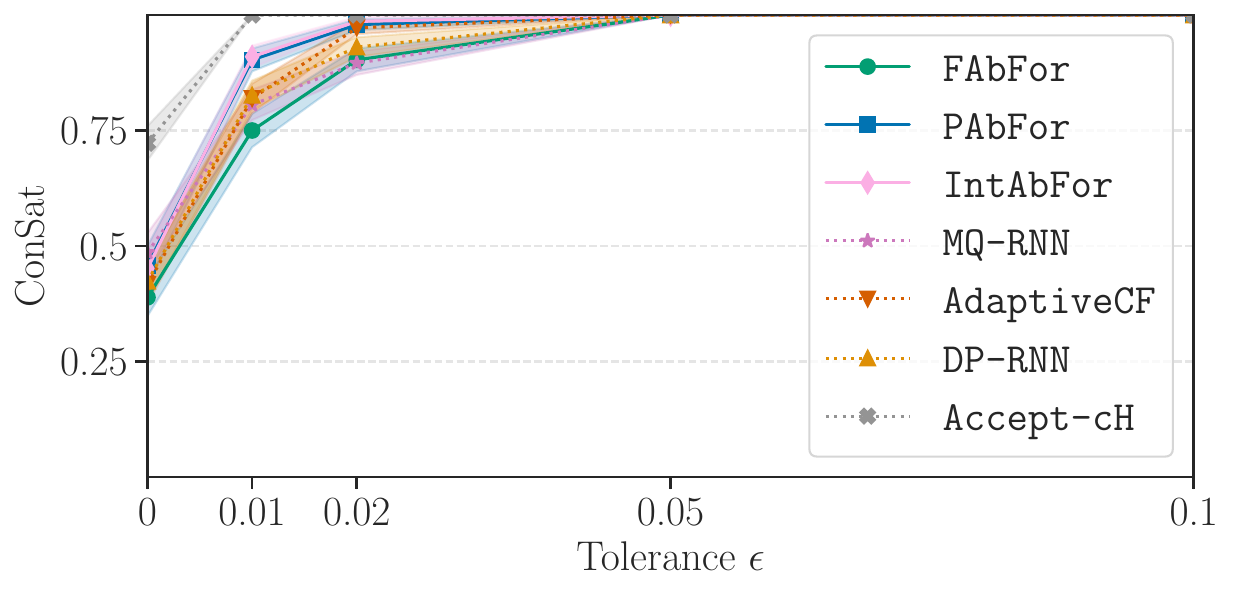}
  \caption{
 Constraint satisfaction ($ConSat$) as a function of the tolerance $\epsilon$ for our three strategies and the baselines. 
 }
  \label{fig:Q3}
\end{figure}
\Figref{fig:Q3} reports the average coverage satisfaction $ConSat$ across five tolerances $\epsilon$ for all considered  methods. In the full abstention setting, all strategies exhibit comparable constraint satisfaction rates, consistently meeting the constraint when $\epsilon \ge 0.05$. Notably, \secondmethod and \thirdmethod outperform all other strategies (with the exception of \dummy, which is specifically designed to exactly meet the constraint). This suggests that the finer-grained control of these abstention strategies, \ie the ability to accept specific intervals rather than making an accept/reject decision on the entire horizon, enables more precise control over the coverage constraint.

%% file: texFiles/5_relatedWork.tex
\section{Related Work}

\textbf{Selective Prediction.} 
\textit{Bounded-abstention} (or \textit{Selective Prediction)}~\citep{DBLP:journals/jmlr/El-YanivW10,geifman2019sectivenet} methods aim to solve the trade-off between selective risk and coverage by learning a predictor together with a selection function that abstains when the uncertainty is high~\citep{hendrickx2024machine,DBLP:conf/aaai/RuggieriP25}. 
However, selective predictors have been predominantly studied in the context of classification \citep{geifman2017selective,DBLP:conf/aistats/PugnanaR23,franc2023optimal} or regression~\citep{zaoui2020regression},
while limited attention has been devoted to other ML tasks~\citep{DBLP:conf/nips/ChengHWW12,perini2023unsupervised,stradiotti2025learning}. 

Only a few works have studied abstention mechanisms for time-series inputs:
\citet{DBLP:conf/ciel/HatamiC13} consider the problem of early-time series classification, while
\citet{DBLP:conf/nips/PidanE11} study selective prediction for single-step trend classification, modeling the temporal trend as a Hidden Markov Model. In contrast, we formulate bounded abstention for multi-horizon forecasting without assuming a specific generative process and theoretically characterize optimal policies for three abstention settings that select which parts of the forecast horizon to output.

\textbf{Uncertainty quantification in time-series forecasting.}
Several paradigms have been proposed for uncertainty quantification in time-series forecasting. \textsc{Bayesian RNNs}~\citep{chien2016bayesian,fortunato2019bayesian,mirikitani2010recursive} capture epistemic uncertainty by placing distributions over model parameters, but they require substantial architectural modifications and depend heavily on the choice of the prior. Simplifications such as Monte Carlo dropout~\citep{gal2016dropout,gal2016theorethically} reduce complexity, but yet often yield poorly calibrated intervals~\citep{alaa2020frequentist}. \textsc{Quantile RNNs}~\citep{wen2018multi} learn prediction intervals directly via the pinball loss, treating the upper and lower bounds as separate prediction targets, but they are prone to quantile crossing and overfitting~\citep{gasthaus2019probabilistic}. More recently, \textsc{Conformal Forecasting RNNs}~\citep{stankeviciute2021conformal} have emerged as a flexible, post-hoc approach for constructing prediction intervals with finite-sample guarantees~\citep{xu2023sequential, auer2023conformal}.

%% file: texFiles/6_conclusion.tex
\section{Conclusions}
We studied learning with abstention in the context of multi-horizon time-series forecasting, which entails learning a model that can simultaneously make predictions for a consecutive sequence of subsequent time points. This required generalizing the notion of abstention from classic regression and classification problems to handle the structured and correlated predictions offered by a multi-horizon forecaster. We formalized the general problem of learning with abstention in multi-horizon forecasting. We proposed three natural ways to instantiate the problem: full abstention, partial abstention, and interval abstention. Each gives progressively more flexibility to the model to select the interval for which it offers a forecast. Using the bounded abstention framework, we derived the theoretical optimal selection function for each setting and provided algorithms learning the corresponding models in practice. Empirically, we found that our proposed approaches significantly outperformed the baselines. Moreover, abstention strategies that have more flexibility over which interval within the horizon the model offers predictions for yielded both lower selective risks and came closer to satisfy the desired constraint on test data.  

\textbf{Limitations.} Our framework assumes that we have access to multiple time series and we can treat them as exchangeable. 
Moreover, our approach trains the estimator of the conditional risk and the predictor jointly. Whether this is possible in real-life scenarios is context-dependent.
Finally, our approach focuses on ambiguity rejection: how to extend the setting also to novelty rejection is left for future works.

%% file: texFiles/appendix.tex
\newpage

\section{Proofs}
\subsection{Proof of \Thref{th:setting1}}\label{app:proof_th1}
\begin{proof}
The selection function in full abstention returns a pair of indices where the start step is fixed at $1$, while the end step can be $0$ (\ie reject the entire horizon) or $H$ (\ie accept the entire horizon):
$$
g_{\texttt{FA}}\left(y_{1:T}\right) = \left(1,e\left(y_{1:T}\right)\right)
$$
Thus, we rewrite the selective risk in \Eqref{eq:sr} for the full abstention setting as:
\begin{equation*}
    \calR(f, g_{\texttt{FA}})=\frac{\bbE_{Y_{1:T+H}}\left[\sum_{t=T+1}^{T+e\left(y_{1:T}\right)} \ell\left(y_t, f_t\left(y_{1: T}\right)\right)\right]}{\mathbb{E}_{Y_{1: T}}\left[e\left(y_{1:T}\right)\right]} = \frac{\bbE_{Y_{1:T+H}}\left[\frac{e\left(y_{1:T}\right)}{H}\sum_{t=T+1}^{T+H} \ell\left(y_t, f_t\left(y_{1: T}\right)\right)\right]}{\mathbb{E}_{Y_{1: T}}\left[e\left(y_{1:T}\right)\right]}
\end{equation*}

Thus, the optimization problem in \Eqref{eq:problem} can be solved by finding the optimal end step $e^*$ as:
    \begin{equation}
    \begin{gathered}
    e^* =  \underset{e}{ \arg\min } \; \frac{\bbE_{Y_{1:T}}\left[\frac{e}{H} r_{1:H}\left(y_{1:T}\right)\right]}{\phi\left(e\right)} \quad 
    \text { s.t. } \quad \phi\left(e\right) \geq cH
    \end{gathered}
    \label{eq:problemproof}
    \end{equation}
Let the total conditional risk over the entire horizon $r_{1:H}\left(y_{1:T}\right)$ be defined as
    \begin{equation*}
        r_{1:H}\left(y_{1:T}\right) = \sum_{t=T+1}^{T+H} \rho_t\left(y_{1:T}\right) = \bbE_{Y_{T+1:T+H} \mid Y_{1:T}}\left[\sum_{t=T+1}^{T+H}\ell\left(y_t, f_t\left(y_{1:T}\right)\right) \mid Y_{1:T} = y_{1:T}\right].
    \end{equation*}

    We will prove the theorem by showing that any other selection function $g_{\texttt{FA}}$ feasible to \Eqref{eq:problemproof} that selects a different $e$ is not an optimal solution. We examine three cases.

    \paragraph{Case 1.} We consider functions where the end step is not always $H$ when $r_{1:H}\left(y_{1:T}\right) < \tau_c^*$. Thus, we define $\bar{g_{\texttt{FA}}}\left(y_{1:T}\right) = (1,\bar{e}\left(y_{1:T}\right))$ such that $\bar{e}\left(y_{1:T}\right)$:
       \begin{equation}
       \begin{aligned}
        \bbE_{Y_{1:T}}\left[\bar{e}\left(y_{1:T}\right)\bbone\left\{r_{1:H}\left(y_{1:T}\right) < \tau_c^*\right\}\right] < \bbE_{Y_{1:T}} \left[e^{*}\left(y_{1:T}\right)\bbone\left\{r_{1:H}\left(y_{1:T}\right) < \tau_c^*\right\}\right]
       \end{aligned}
        \label{eq:cond1violated}
    \end{equation}
        Thus, there is a subset $Y^{\prime}_{1:T} \subseteq \calY_{1:T}$ such that
    \begin{equation*}
        \forall y_{1:T} \in Y^{\prime}_{1:T} : r_{1:H}\left(y_{1:T}\right) \geq \tau_c^*
    \end{equation*}
    and 
    \begin{equation}
        \begin{aligned}
\!\!\bbE_{Y^{\prime}_{1:T}}\left[\bar{e}\left(y_{1:T}\right)\right] = \bbE_{Y_{1:T}}\!\!\left[H\bbone\left\{r_{1:H}\left(y_{1:T}\right) < \tau_c^*\right\}\right] - \bbE_{Y_{1:T}}\!\!\left[\bar{e}\left(y_{1:T}\right)\bbone\left\{r_{1:H}\left(y_{1:T}\right) < \tau_c^*\right\}\right] > 0
        \label{eq:gexpectation}
        \end{aligned}
    \end{equation}
    where the inequality holds for \eqref{eq:cond1violated}. Then, we denote the objective function in \Eqref{eq:problemproof} as 
    \begin{equation*}
        F(e) = \frac{\bbE_{Y_{1:T}}\left[\frac{e\left(y_{1:T}\right)}{H}r_{1:H}\left(y_{1:T}\right)\right]}{\bbE_{Y_{1:T}}\left[e\left(y_{1:T}\right)\right]} = \frac{\bbE_{Y_{1:T}}\left[\frac{e\left(y_{1:T}\right)}{H}r_{1:H}\left(y_{1:T}\right)\right]}{\phi\left(e\right)}
    \end{equation*}
    and define $g_{\texttt{FA}}^{\prime}\left(y_{1:T}\right)$ that selects $\left(1, e^{\prime}\left(y_{1:T}\right)\right)$ as
    \begin{equation*}
        g_{\texttt{FA}}^{\prime}(y_{1:T})=\left\{\begin{array}{ccc}
(1,H) & \text { if } & r_{1:H}\left(y_{1:T}\right)<\tau_c^*, \\
(1,0) & \text { if } & y_{1:T} \in Y^{\prime}_{1:T}, \\
g_{\texttt{FA}}\left(y_{1:T}\right) &  & \text { otherwise }
\end{array}\right.
    \end{equation*}
    where $g_{\texttt{FA}}^{\prime}$ is a solution to \eqref{eq:problemproof} since $\phi(\bar{e}) = \phi\left(e^{\prime}\right)$. Now, if we compute 
    \begin{equation*}
    \begin{aligned}
        \phi\left(\bar{e}\right)\left(F\left(\bar{e}\right)  - F\left(e^{\prime}\right)\right) &= \bbE_{Y^{\prime}_{1:T}}\left[r_{1:H}\left(y_{1:T}\right) \frac{\bar{e}\left(y_{1:T}\right)}{H}\right] - \bbE_{Y_{1:T}}\left[r_{1:H}(y_{1:T})\bbone\left\{r_{1:H}\left(y_{1:T}\right) < \tau^*_c\right\}\right] \\
        & \qquad + \bbE_{Y_{1:T}}\left[r_{1:H}\left(y_{1:T}\right)\frac{\bar{e}\left(y_{1:T}\right)}{H}\bbone\left\{r_{1:H}\left(y_{1:T}\right) < \tau^*_c\right\}\right]\\
        & \geq \bbE_{Y_{1:T}} \left[\tau^*_c \left(1-\frac{\bar{e}\left(y_{1:T}\right)}{H}\right)\bbone\left\{r_{1:H}\left(y_{1:T}\right) < \tau_c^*\right\}\right] \\
        & \qquad - \bbE_{Y_{1:T}} \left[r_{1:H}\left(y_{1:T}\right)\left(1-\frac{\bar{e}\left(y_{1:T}\right)}{H}\right)\bbone\left\{r_{1:H}\left(y_{1:T}\right) < \tau_c^*\right\}\right] \\
        & \geq \bbE_{Y_{1:T}} \left[\left(\tau^*_c - r_{1:H}\left(y_{1:T}\right)\right)\left(1 - \frac{\bar{e}\left(y_{1:T}\right)}{H}\right)\bbone\left\{r_{1:H}\left(y_{1:T}\right) < \tau^*_c\right\}\right] > 0
    \end{aligned}
    \end{equation*}
Where the last inequality comes from linearity of the expectation given \eqref{eq:gexpectation} and $\left(\tau^*_c - r_{1:H}\left(y_{1:T}\right)\right) > 0$ over the set $\bbone\left\{r_{1:H}\left(y_{1:T}\right) < \tau_c^*\right\}$. Since, this value is greater than zero we have that $\bar{g_{\texttt{FA}}}$ is not an optimal solution.
    \paragraph{Case 2.} We want to show that for any other function $\bar{g_{\texttt{FA}}} = \left(1, \bar{e}\right)$ as
    \begin{equation*}
        \bar{g_{\texttt{FA}}}(y_{1:T})=\left\{\begin{array}{ccc}
 g_{\texttt{FA}}^{*}(y_{1:T}) & \text { if } & r_{1:H}\left(y_{1:T}\right) \neq \tau^*_c, \\
(1, \bar{\xi}) & \text { if } & r_{1:H}\left(y_{1:T}\right) = \tau^*_c \\
\end{array}\right.
    \end{equation*}
    where $\bar{\xi}$ is a random variable such that $\bar{\xi} = H$ with probability $\bar{\kappa}$ and $\bar{\xi} = 0$ with probability $1-\bar{\kappa}$ with $\kappa \ne \bar{\kappa}$ , then $F\left(\bar{e}\right) > F\left(e^*\right)$.
    
    We distinguish two sub-cases:
    \begin{enumerate}[label=\alph*.]
        \item $\bar{\kappa} > \kappa$. Then, 
            \begin{equation*}
        \begin{aligned}
         \bbE_{Y_{1:T}} \left[\bar{e}\left(y_{1:T}\right)\right]
           &=  \bbE_{Y_{1:T}} \left[\bar{e}\left(y_{1:T}\right)\bbone\left\{r_{1:H}\left(y_{1:T}\right) < \tau^*_c\right\}\right] +  \bbE_{Y_{1:T}} \left[\bar{e}\left(y_{1:T}\right)\bbone\left\{r_{1:H}\left(y_{1:T}\right) = \tau^*_c\right\}\right]  \\
           &> \bbE_{Y_{1:T}} \left[e^*\left(y_{1:T}\right)\bbone\left\{r_{1:H}\left(y_{1:T}\right) < \tau^*_c\right\}\right] + \bbE_{Y_{1:T}} \left[\xi^*\bbone\left\{r_{1:H}\left(y_{1:T}\right) = \tau^*_c\right\}\right] = cH
        \label{eq:ggeqcoverage}
        \end{aligned}
    \end{equation*}
        Thus, have that $\phi\left(\bar{e}\right) > cH$. Thus, we define another function $g_{\texttt{FA}}^{\prime}=(1, e^{\prime})$ such that $e^{\prime} = \frac{cH}{\phi\left(\bar{e}\right)}\bar{e}$ and $\phi\left(e^{\prime}\right)=cH$. $g_{\texttt{FA}}^{\prime}$ is a feasible solution for \Eqref{eq:problemproof}, but
    \begin{equation*}
        \begin{aligned}
        \bbE_{Y_{1:T}}\left[e^{\prime}\left(y_{1:T}\right)\bbone\left\{r_{1:H}\left(y_{1:T}\right) < \tau^*_c\right\}\right]   &= \frac{cH}{\phi\left(\bar{e}\right)}\bbE_{Y_{1:T}} \left[\bar{e}\left(y_{1:T}\right)\bbone\left\{r_{1:H}\left(y_{1:T}\right)< \tau^*_c\right\}\right]\\
        &< \bbE_{Y_{1:T}} \left[e^*\left(y_{1:T}\right)\bbone\left\{r_{1:H}\left(y_{1:T}\right) < \tau^*_c\right\}\right]
        \end{aligned}
    \end{equation*}
    As seen in Case 1, $g_{\texttt{FA}}^{\prime}$ is not an optimal solution. Since $F\left(\bar{e}\right) = F\left(e^{\prime}\right)$, this implies that $\bar{g_{\texttt{FA}}}$ is not an optimal solution too.
    \item $\bar{\kappa} < \kappa$. In this case, there is a $g_{\texttt{FA}}^{\prime} = \left(1, e^{\prime}\right)$ such that:
        \begin{equation*}
        g_{\texttt{FA}}^{\prime}(y_{1:T})=\left\{\begin{array}{ccc}
 \bar{g_{\texttt{FA}}}(y_{1:T}) & \text { if } & r_{1:H}\left(y_{1:T}\right) < \tau^*_c, \\
(1,0) & \text { if } & r_{1:H}\left(y_{1:T}\right) > \tau^*_c, \\
\end{array}\right.
    \end{equation*}
    and 
        \begin{equation*}
        \begin{aligned}
        \bbE_{Y_{1:T}} \left[e^{\prime}\left(y_{1:T}\right)\bbone\left\{r_{1:H}\left(y_{1:T}\right) = \tau^*_c\right\}\right] &= \bbE_{Y_{1:T}}\left[\bar{e}\left(y_{1:T}\right)\bbone\left\{r_{1:H}\left(y_{1:T}\right) = \tau_c^*\right\}\right] \\
        & \qquad + \bbE_{Y_{1:T}}\left[\bar{e}\left(y_{1:T}\right)\bbone\left\{r_{1:H}\left(y_{1:T}\right) > \tau^*_c\right\}\right] 
        \end{aligned}
    \end{equation*}
    We have that  $\phi\left(\bar{e}\right) = \phi\left(e^{\prime}\right)$. Now, we derive:
    \begin{equation*}
        \begin{aligned}
            \phi\left(\bar{e}\right)\left(F\left(\bar{e}\right) - F\left(e^{\prime}\right)\right) &= \bbE_{Y_{1:T}} \left[\frac{\bar{e}\left(y_{1:T}\right)}{H}r_{1:H}\left(y_{1:T}\right) - \frac{e^{\prime}\left(y_{1:T}\right)}{H}r_{1:H}\left(y_{1:T}\right)\right] \\
            &=  \bbE_{Y_{1:T}} \left[\frac{\bar{e}\left(y_{1:T}\right)}{H}\left(y_{1:T}\right)r_{1:H}\left(y_{1:T}\right)\bbone\left\{r_{1:H}\left(y_{1:T}\right) > \tau^*_c\right\}\right] \\
            & \qquad - \tau^*_c \bbE_{Y_{1:T}} \left[\frac{\bar{e}\left(y_{1:T}\right)}{H}\bbone\left\{r_{1:H}\left(y_{1:T}\right) > \tau^*_c\right\}\right] \\ 
            &=  \bbE_{Y_{1:T}} \left[\frac{\bar{e}\left(y_{1:T}\right)}{H}\underbrace{\left(r_{1:H}\left(y_{1:T}\right)-\tau^*_c\right)}_{>0}\bbone\left\{r_{1:H}\left(y_{1:T}\right) > \tau^*_c\right\}\right]
            > 0\\
        \end{aligned}
    \end{equation*}
    Where, again, the last inequality comes from linearity of the expectation given $\bar{e}\left(y_{1:T}\right) > 0$ and $r_{1:H}\left(y_{1:T}\right) > \tau_c^*$ over the set $\bbone\left\{r_{1:H}\left(y_{1:T}\right) > \tau^*_c\right\}$. Thus, $\bar{g_{\texttt{FA}}}$ is not an optimal solution.
    \end{enumerate}

    \paragraph{Case 3.} In the final case, we consider functions $\bar{g_{\texttt{FA}}} = (1, \bar{e})$ that do not always return $(1,0)$ when $r_{1:H}\left(y_{1:T}\right)> \tau_c^*$. In this case, we have that $\phi\left(\bar{e}\right) > cH$. As shown in Case 2a, this can not be an optimal solution.
\end{proof}

\subsection{Proof of \Thref{th:setting2}}\label{app:proof_th2}
We start by introducing two lemmas that will be used later in the proof.
\begin{lemma}
    \label{lem:fractional_to_linear}
Assume $D(e) > 0$ for any $e \in \mathcal{E}^{\prime}$ and $\calR(e) = \frac{N(e)}{D(e)}$. Then minimizing $\calR(e)$ is equivalent to solving the following parametric linearized problem:
\begin{equation*}
e^* \in \underset{e \in \mathcal{E}'}{\arg \min} \frac{N(e)}{D(e)} \; \Longleftrightarrow \; e^* \in \underset{e \in \calE^{\prime}}{\arg \min} \big( N(e) - \lambda^* D(e) \big)
\end{equation*}
where $\lambda^* = \calR(e^*) = \frac{N(e^*)}{D(e^*)} \ge 0$.
\end{lemma}
\begin{proof}
    We prove each direction separately.
    \begin{itemize}
    \item[$\Rightarrow$]
    Let $e^* \in \underset{e \in \calE^{\prime}}{\arg\min}~\calR(e)$, \ie
    \[
        \frac{N(e^*)}{D(e^*)} \le \frac{N(e)}{D(e)}, \qquad \forall e \in \calE^{\prime}.
    \]
    Since $D(e) > 0$, we can rewrite this inequality as 
    \[
        N(e) - \frac{N(e^*)}{D(e^*)} D(e) \ge 0
    \]
    Let $\lambda^* = \frac{N(e^*)}{D(e^*)}$ be the risk of the optimal solution, then the inequality simplifies to:
    \[
    N(e) - \lambda^* D(e) \ge 0  
    \]
    Evaluating this parametric form at the optimal solution $e^*$, we obtain 
    \[
     N(e) - \lambda^* D(e) \geq  0 = N(e^*) - \frac{N(e^*)}{D(e^*)}D(e^*)  =  N(e^*) - \lambda^* D(e^*).
    \]
    Hence, there exists $\lambda^*$ such that $e^*$ also minimizes the parametric form $N(e) - \lambda^* D(e)$.
    \item[$\Leftarrow$]
    Let $\lambda^* = \frac{N(e^*)}{D(e^*)}$, we have that
    \[
        N(e^*) - \lambda^* D(e^*) \le N(e) - \lambda^* D(e),
        \qquad \forall e \in \calE^{\prime}.
    \]
    Dividing both sides by $D(e) > 0$, we obtain
    \[
        \lambda^* = \frac{N(e^*)}{D(e^*)} \le \frac{N(e)}{D(e)} \qquad \forall e \in \calE^{\prime},
    \]
    so $e^*$ also minimizes the fractional objective.
    \end{itemize}
\end{proof}

\begin{lemma}
    Given $\gamma > 0$, define for a fixed $y_{1:T}$
    \[
        e(\gamma) := \underset{e \in \{0,\dots,H\}}{\arg\min} \bigl\{ r_{1:e}(y_{1:T}) - \gamma e \bigr\},
    \]
    \[
        D(\gamma) := \mathbb{E}_{Y_{1:T}}\bigl[e(\gamma)\bigr].
    \]
    where we select the smallest minimizer in case of ties. Then for any fixed $y_{1:T}$, $e(\gamma)$ is non-decreasing in $\gamma$. Consequently, $D(\gamma)$ is also non-decreasing in $\gamma$.
    \label{lem:monotonicity_gamma}
\end{lemma}
\begin{proof}
    We have to prove that for a given $y_{1:T}$: \[ \gamma_2 > \gamma_1 \Rightarrow e(\gamma_2) \ge e(\gamma_1) \] Let $e_1 := e(\gamma_1)$ and $e_2 := e(\gamma_2)$. By optimality of $e_1$ at $\gamma_1$ and $e_2$ at $\gamma_2$, we have:
    \[
    r_{1:e_1} - \gamma_1 e_1 \le r_{1:e_2} - \gamma_1 e_2,
    \qquad
    r_{1:e_2} - \gamma_2 e_2 \le r_{1:e_1} - \gamma_2 e_1.
    \]
    Summing these inequalities gives
    \[
        (\gamma_1 - \gamma_2)(e_2 - e_1) \le 0.
    \]
    Since $\gamma_2 > \gamma_1$, the first factor is negative, hence $e_2 - e_1 \ge 0$, \ie\ $e_2 \ge e_1$.
    Linearity and monotonicity of expectation imply that $D(\gamma)$ is also non-decreasing.
\end{proof}

We can now prove the theorem.
\begin{proof}
We define the cumulative conditional risk at step $e$ as for a time series $y_{1:T}$ :
\begin{equation*}
    r_{1:e}\left(y_{1:T}\right) = \sum_{t=T+1}^{T+e} \rho_t\left(y_{1:T}\right) = \bbE_{Y_{T+1:T+H} \mid Y_{1:T}}\left[\sum_{t=T+1}^{T+e}\ell\left(y_t, f_t\left(y_{1:T}\right)\right) \mid Y_{1:T} = y_{1:T}\right],
    \end{equation*}

In the partial abstention setting, the start step is fixed at $s=1$. Thus, the selection function $g_{\texttt{PA}}$ needs only to determine the end step $e \in \{0, \dots, H\}$. The length of the forecast is $h(g_{\texttt{PA}}) = e$. The optimization problem in~\Eqref{eq:problem} can then be rewritten as:
\begin{equation*}
    \begin{aligned}
    \underset{e}{\min}\; \frac{\bbE_{Y_{1:T}} \left[  r_{1:e\left(y_{1:T}\right)}\left(y_{1:T}\right) \right]  }{\bbE_{Y_{1:T}}\left[e\left(y_{1:T}\right)\right]} \quad \text{s.t.} \quad \bbE_{Y_{1:T}}\left[e\left(y_{1:T}\right)\right] \ge cH
\end{aligned}
\end{equation*}

Let $N(e) := \bbE_{Y_{1:T}} \left[  r_{1:e\left(y_{1:T}\right)}\left(y_{1:T}\right)\right]$ and $D(e) := \bbE_{Y_{1:T}}\left[e\left(y_{1:T}\right)\right]$. Given the set of admissible selection functions $\calE^{\prime} = \{e: D(e) \geq cH\}$, by \Lemref{lem:fractional_to_linear} the objective to minimize becomes:
\begin{equation*}
    \underset{e \in \calE^{\prime}}{\min} \left\{N(e) - \lambda^* D(e)\right\},
\end{equation*}
or, by making the coverage constraint explicit:
\begin{equation*}
    \underset{e}{\min} \left\{N(e) - \lambda^* D(e)\right\} \quad \text{s.t.} \quad D(e) \ge cH.
\end{equation*}
Solving this is equivalent to the following parametric Lagrangian problem:
\begin{equation*}
   \begin{aligned}
    \calL(e,\lambda^*,\eta) = N(e) - \lambda^* D(e) + \eta (cH - D(e)) = N(e) - (\lambda^* + \eta) D(e) + \eta cH,
\end{aligned} 
\end{equation*}
where $\lambda^*$ comes from the linearized transformation in \Lemref{lem:fractional_to_linear} and $\eta\ge 0$ is the Lagrange multiplier associated with the coverage constraint.

Let $\gamma := \lambda + \eta$, the expectation is linear and $e$ operates to each time-series $y_{1:T}$ independently. Consequently, minimizing the objective with respect to $e$ is equivalent to selecting the end step that minimizes each local term independently:
 \begin{equation*}
 \begin{aligned}
     \underset{e}{\min} \; \bbE_{Y_{1:T}}\left[r_{1:e\left(y_{1:T}\right)}\left(y_{1:T}\right) -  \gamma e\left(y_{1:T}\right)\right] =  \bbE_{Y_{1:T}}\left[\underset{e \in \left\{0,\dots,H\right\}}{\min} \left(r_{1:e}\left(y_{1:T}\right) -  \gamma e\right)\right] 
 \end{aligned}
 \end{equation*}
Thus, for each sequence $y_{1:T}$, the optimal local decision is
\begin{equation*}
    e^*\left(y_{1:T}\right) = \underset{e \in \left\{0,\dots,H\right\}}{\arg \min} \left\{r_{1:e}\left(y_{1:T}\right) - \gamma e\right\}.
\end{equation*}
By \Lemref{lem:monotonicity_gamma}, the expected coverage $D({\gamma}) = \bbE_{Y_{1:T}}\left[e_{\gamma}(y_{1:T})\right]$ is non-decreasing in $\gamma$ and
\[
    \lim_{\gamma \to 0^+} D({\gamma})= 0, \qquad
    \lim_{\gamma \to +\infty} D({\gamma}) = H.
\]
Since we assume the joint vector of the conditional variances is absolutely continuous, the expected coverage function $D(\gamma)$ is absolutely continuous with respect to $\gamma$. By the Intermediate Value Theorem, and given the boundary conditions, there exists a $\gamma$ such that $D(\gamma)=cH$ exactly. 

Let $\eta^*$ be optimal dual variable associated with the constrained fractional problem, and let $\gamma^* = \lambda^* + \eta^*$. By construction, $e^* := e({\gamma^*})$ minimizes the Lagrangian $\calL(e, \eta^*)$.
The complementary slackness condition gives
\[
    \eta^* \left(cH - D(e^*)\right) = 0.
\]
Hence, two cases arise:
\begin{itemize}
    \item if $D(e^*) = cH$, the coverage constraint is active and $\eta^* \ge 0$ may be positive.
    In this case, $\lambda^* = \gamma^* - \eta^*$ and, by definition of the fractional optimum,
    $\lambda^* = \calR(e^*)$.
    \item if $D(e^*) > cH$, the constraint is inactive and $\eta^* = 0$,
    so that $\gamma^* = \lambda^*$.
\end{itemize}
By solving the Lagrangian with complementary slackness, we obtain $e^*$ that minimizes the linearized parametric problem under the constraint $\bbE_{Y_{1:T}}\left[e\left(y_{1:T}\right)\right] \ge cH$. By the equivalence established above, this $e^*$ also minimizes the original fractional objective under the same constraint. Therefore, $e^*$ is the solution to the constrained fractional problem.

Finally, the optimal selection function $g_{\texttt{PA}}$ is obtained as:
\[
g_{\texttt{PA}}^*\left(y_{1:T}\right) = \left(1, e^*\left(y_{1:T}\right)\right)
\]
\end{proof}

\subsection{Proof of Theorem 3}\label{app:proof_th3}
\begin{proof}
We first introduce an alternative formulation for the selection function $\bar{g_{\texttt{IA}}}(y_{1:T}) = (s(y_{1:T}), h(y_{1:T}))$, where $s(y_{1:T})$ is the start step and $h(y_{1:T})$ is the length of the forecast interval. Solving the problem in \Eqref{eq:problem} with respect to $g_{\texttt{IA}}$ is equivalent to solving the same problem with respect to $\bar{g_{\texttt{IA}}}$, since the first can be obtained from the second just by computing the end step as $e(y_{1:T}) = s(y_{1:T}) + h(y_{1:T}) - 1$. Thus, in the following, we will obtain the optimal $\bar{g_{\texttt{IA}}}^*$, from which the optimal $g_{\texttt{IA}}^*$ can be directly derived.

We denote the cumulative conditional risk for a forecast interval starting at step $s$ with interval length $h$ as:
\begin{equation*}
    r_{s:(s+h-1)}(y_{1:T}) = \sum_{t=T+s}^{T+s+h-1} \rho_t(y_{1:T}) = \bbE_{Y_{T+1:T+H} \mid Y_{1:T}}\left[\sum_{t=T+s}^{T+s+h-1}\ell\left(y_t, f_t\left(y_{1:T}\right)\right) \mid Y_{1:T} = y_{1:T}\right].
\end{equation*}
Given $\bar{g_{\texttt{IA}}}(y_{1:T}) = (s(y_{1:T}), h(y_{1:T}))$, the optimization problem in this setting can be written as:
\begin{equation*}
    \begin{aligned} 
    \underset{s,h}{\min}\frac{\bbE_{Y_{1:T}} \left[  r_{s:s+h-1}\left(y_{1:T}\right) \right]  }{\bbE_{Y_{1:T}}\left[h\right]} \quad \text{s.t.} \quad \bbE_{Y_{1:T}}\left[h\right] \ge cH
\end{aligned}
\end{equation*}
Following the result of \Lemref{lem:fractional_to_linear}, this constrained fractional problem is equivalent to minimizing the linearized Lagrangian objective:
\begin{equation*}
    \mathcal{L}(s, h, \gamma) = \mathbb{E}_{Y_{1:T}} \left[ r_{s:(s+h-1)}(y_{1:T})  - \gamma \cdot h \right] + \eta cH,
\end{equation*}
where $\gamma = \lambda^* + \eta$, $\lambda^*$ comes from the linearized transformation in \Lemref{lem:fractional_to_linear} and $\eta\ge 0$ is the Lagrange multiplier associated with the coverage constraint.

Since the expectation is linear and the optimization is made independently for each time series, we can minimize the inner term point-wise. Thus, for a fixed $\gamma$, the optimal selection strategy $\bar{g_{\texttt{IA}}}^*(y_{1:T}) = (s^*\left(y_{1:T}\right), h^*\left(y_{1:T}\right))$ is the minimizer of the local cost:
\begin{equation*}
    (s^*\left(y_{1:T}\right), h^*\left(y_{1:T}\right)) = \underset{\substack{s \in \{1, \dots, H\} \\ h \in \{0, \dots, H-s+1\}}}{\arg \min} \left( \sum_{t=T+s}^{T+s+h-1} \rho_t(y_{1:T}) - \gamma h \right).
\end{equation*}
where $(s^*\left(y_{1:T}\right), h^*\left(y_{1:T}\right)) = (1,0)$ means the entire horizon is rejected and the objective value is 0.
This optimization can be decomposed into two steps:
\begin{enumerate}
    \item \textbf{Best start for fixed length:} for each time series $y_{1:T}$, we compute a candidate start time $s_h\left(y_{1:T}\right)$ for every possible interval length $h \in \{0, \dots, H\}$ that minimizes the conditional risk for that specific duration:
    $$ s^*_h(y_{1:T}) = \underset{s \in \{1, \dots, H-h+1\}}{\arg \min} \sum_{t=T+s}^{T+s+h-1} \rho_t(y_{1:T}) $$
    \item \textbf{Optimal length:} Select the optimal duration $h^*$ that balances the minimal conditional risk against the length reward $\gamma$:
    $$ h^*\left(y_{1:T}\right) = \underset{h \in \{0, \dots, H\}}{\arg \min} \left( \sum_{t=T+s_h\left(y_{1:T}\right)}^{T+s_h\left(y_{1:T}\right)+h-1} \rho_t(y_{1:T}) - \gamma h \right) $$
\end{enumerate}
As in Theorem 2, under the assumption that the joint vector of conditional risks is absolutely continuous, the expected length $\mathbb{E}[h^*_{\gamma}\!\left(y_{1:T}\right)]$ is a continuous, non-decreasing function of $\gamma$. Thus, there exists a $\gamma$ such that the constraint $\mathbb{E}[h^*\!\left(y_{1:T}\right)] = c H$ is satisfied exactly.

Finally, the optimal selection function $g^*_{\texttt{IA}}$ is obtained as:
\[
    g^*_{\texttt{IA}}\left(y_{1:T}\right) = \left(s^*_{h^*}\left(y_{1:T}\right), s^*_{h^*}\left(y_{1:T}\right)+h^*\left(y_{1:T}\right) -1\right)
\]
\end{proof}

\section{Further discussion on continuity assumption} 
In \Thref{th:setting2} and \Thref{th:setting3}, we assume that the joint distribution of the conditional risk vector is absolutely continuous with respect to the Lebesgue measure. This assumption is required to guarantee the existence of a deterministic $\gamma$ that satisfies the target coverage exactly. While we acknowledge that this continuity may not hold in practice, it does not limit the practical applicability of our framework. As detailed in \Secref{sec:implementation}, the optimal strategy can be implemented via a randomized policy to handle discrete jumps in coverage. This guarantees exact constraint satisfaction and demonstrates robust performance in our empirical evaluation (see \Secref{sec:results}).

\section{Algorithmic implementations}\label{app:algorithms}
\subsection{\secondmethod}\label{app:algorithm_pabfor}
Algorithm~\ref{alg:pabfor} provides the full details for \secondmethod outlined in \Secref{sec:implementation} (Partial abstention). The algorithm first trains the forecaster and the conditional variance estimator on the training set, then it estimates $\hat{\gamma}_{\ell}$, $\hat{\gamma}_{r}$, and $p$ using the calibration set.
\begin{algorithm}
\caption{\secondmethod ($\calD_{\text{train}}$, $\calD_{\text{calib}}$, $c$, $H$, $\epsilon$)}
\textbf{Inputs}: $\calD_{\text{train}}$ - training data, $\calD_{\text{calib}}$ - calibration data, $c$ - target coverage, $H$ - Forecast horizon, $\epsilon$ - tolerance for $\gamma$

\textbf{Output}: $\hat{\gamma}_\ell$, $\hat{\gamma}_r$ - the lower and upper bounds for $\gamma$, $p$ - the probability to select the policy
\begin{algorithmic}[1]
\State for = Forecaster(); \; for.fit($\calD_{\text{train}}$)
\State $\hat{f}_{\text{calib}}$, $\hat{\sigma}^2_{\text{calib}}$ = for.predict($\calD_{\text{calib}}$)
\State $\hat{r}$ = $cumSum$($\hat{\sigma}^2_{\text{calib}}$, axis=0)
\State $\hat{r}$ = $hstack$([0,...0], r) \qquad  \qquad \qquad \qquad \qquad //  to account for the rejection of the full horizon
\State $\hat{\gamma}_{\ell} = 0$, $\hat{\gamma}_{r} = max(r)$ 
\While{$\hat{\gamma}_{r} - \hat{\gamma}_{\ell} > \epsilon$} 
    \State $\hat{\gamma} = .5 \left(\hat{\gamma}_{r} + \hat{\gamma}_{\ell}\right)$
    \State $\hat{e} = \left[\text{argmin}_e \left(\hat{r}(y)[e] - \hat{\gamma} e\right) \; \textbf{for} \; y \in \calD_{\text{calib}}\right]$
    \State $\hat{\phi}_{\hat{\gamma}} = mean(\hat{e})$

    \If{ $\hat{\phi}_{\hat{\gamma}} = cH$}
        \State $\hat{\gamma}_{\ell} = \hat{\gamma}_{r} = \hat{\gamma}$  \qquad \qquad \qquad // if coverage meets the target level, stop the binary search
    \ElsIf{$\hat{\phi}_{\hat{\gamma}} < cH$}
        \State $\hat{\gamma}_{\ell} = \hat{\gamma}$  \qquad \qquad \qquad \qquad \qquad \qquad // if coverage is below the target level, increase $\hat{\gamma}$
    \ElsIf{$\hat{\phi}_{\hat{\gamma}} > cH$}
        \State $\hat{\gamma}_{r} = \hat{\gamma}$ \qquad \qquad \qquad \qquad \qquad \qquad // if coverage is above the target level, decrease $\hat{\gamma}$
    \EndIf
\EndWhile
\State $\hat{e}_{\hat{\gamma}_{\ell}} = \left[\text{argmin}_e \left(\hat{r}(y)[e] - \hat{\gamma}_{\ell} e\right) \; \textbf{for} \; y \in \calD_{\text{calib}}\right]$
\State $\hat{e}_{\hat{\gamma}_{r}} = \left[\text{argmin}_e \left(\hat{r}(y)[e] - \hat{\gamma}_{r} e\right) \; \textbf{for} \; y \in \calD_{\text{calib}}\right]$
\State $\hat{\phi}_{\hat{\gamma_{\ell}}} = mean(\hat{e}_{\hat{\gamma_{\ell}}})$, $\hat{\phi}_{\hat{\gamma_{r}}} = mean(\hat{e}_{\hat{\gamma_{r}}})$
\State $p = \frac{cH - \hat{\phi}_{\hat{\gamma_{r}}}}{\hat{\phi}_{\hat{\gamma_{\ell}}} - \hat{\phi}_{\hat{\gamma_{r}}}}$
\State \textbf{return} $\hat{\gamma}_{\ell}$,  $\hat{\gamma}_r$, $p$
\end{algorithmic}
\label{alg:pabfor}
\end{algorithm}

\subsection{\thirdmethod}\label{app:algorithm_intabfor}
Algorithm~\ref{alg:intabfor} provides the full details for \thirdmethod outlined in \Secref{sec:implementation} (Interval abstention). The algorithm first trains the forecaster and the conditional variance estimator on the training set. Next, for each time-series in the calibration set, it identifies the subsequence with the lowest cumulative variance for every possible interval length. Finally, it estimates $\hat{\gamma}_{\ell}$, $\hat{\gamma}_{r}$, and $p$ considering only these optimal subsequences.
\begin{algorithm}
\caption{\thirdmethod ($\calD_{\text{train}}$, $\calD_{\text{calib}}$, $c$, $H$, $\epsilon$)}
\textbf{Inputs}: $\calD_{\text{train}}$ - training data, $\calD_{\text{calib}}$ - calibration data, $c$ - target coverage, $H$ - Forecast horizon, $\epsilon$ - tolerance for $\gamma$

\textbf{Output}: $\hat{\gamma}_\ell$, $\hat{\gamma}_r$ - the lower and upper bounds for $\gamma$, $p$ - the probability to select the policy
\begin{algorithmic}[1]
\State for = Forecaster(); \; for.fit($\calD_{\text{train}}$)
\State $\hat{f}_{\text{calib}}$, $\hat{\sigma}^2_{\text{calib}}$ = for.predict($\calD_{\text{calib}}$)
\For{h in $\{..H\}$} 
    \State $\hat{s}[h] = \left[\text{argmin}_{s} \left(\hat{\sigma}^2(y)[s\!:\!s+h] \right) \; \textbf{for} \; y \in \calD_{\text{calib}}\right]$ \; // compute the start of the best interval
\EndFor
\State $\hat{\gamma}_{\ell} = 0$, $\hat{\gamma}_{r} = max(r)$ 
\While{$\hat{\gamma}_{r} - \hat{\gamma}_{\ell} > \epsilon$} 
    \State $\hat{\gamma} = .5 \left(\hat{\gamma}_{r} + \hat{\gamma}_{\ell}\right)$
    \State $\hat{h} = \left[\text{argmin}_h \left(\hat{\sigma}^2(y)\left[\hat{t}_s[h](y)\!:\!\hat{t}_s[h](y)\!+\!h\right] - \hat{\gamma} h\right) \; \textbf{for} \; y \in \calD_{\text{calib}}\right]$
    \State $\hat{\phi}_{\hat{\gamma}} = mean(\hat{h})$

    \If{ $\hat{\phi}_{\hat{\gamma}} = cH$}
        \State $\hat{\gamma}_{\ell} = \hat{\gamma}_{r} = \hat{\gamma}$ \qquad \qquad \qquad // if coverage meets the target level, stop the binary search
    \ElsIf{$\hat{\phi}_{\hat{\gamma}} < cH$}
        \State $\hat{\gamma}_{\ell} = \hat{\gamma}$ \qquad \qquad \qquad \qquad \qquad \qquad // if coverage is below the target level, increase $\hat{\gamma}$
    \ElsIf{$\hat{\phi}_{\hat{\gamma}} > cH$}
        \State $\hat{\gamma}_{r} = \hat{\gamma}$  \qquad \qquad \qquad \qquad \qquad \qquad // if coverage is above the target level, decrease $\hat{\gamma}$
    \EndIf
\EndWhile
\State $\hat{h}_{\hat{\gamma}_{\ell}} = \left[\text{argmin}_h \left(\hat{\sigma}^2(y)\left[\hat{s}[h](y)\!:\!\hat{s}[h](y)\!+\!h\right] - \hat{\gamma}_{\hat{\gamma}_{\ell}} h\right) \; \textbf{for} \; y \in \calD_{\text{calib}}\right]$
\State $\hat{h}_{\hat{\gamma}_{r}} = \left[\text{argmin}_h \left(\hat{\sigma}^2(y)\left[\hat{s}[h](y)\!:\!\hat{s}[h](y)\!+\!h\right] - \hat{\gamma}_{\hat{\gamma}_{r}} h\right) \; \textbf{for} \; y \in \calD_{\text{calib}}\right]$
\State $\hat{\phi}_{\hat{\gamma_{\ell}}} = mean(\hat{h}_{\hat{\gamma_{\ell}}})$, $\hat{\phi}_{\hat{\gamma_{r}}} = mean(\hat{h}_{\hat{\gamma_{r}}})$
\State $p = \frac{cH - \hat{\phi}_{\hat{\gamma_{r}}}}{\hat{\phi}_{\hat{\gamma_{\ell}}} - \hat{\phi}_{\hat{\gamma_{r}}}}$
\State \textbf{return} $\hat{\gamma}_{\ell}$,  $\hat{\gamma}_r$, $p$
\end{algorithmic}
\label{alg:intabfor}
\end{algorithm}
\section{Experimental setup : additional details} \label{app:experiments}
\subsection{Data}\label{app:data}
Table~\ref{tab:datasets} summarizes the characteristics of the 24 publicly available datasets used in our evaluation. Prior to training, we apply min-max normalization to the data. Detailed descriptions of each dataset and specific preprocessing steps are provided in the following paragraphs.
\begin{table*}[ht]
\centering
\caption{Characteristics of the five real-world datasets and $19$ benchmark datasets used for the experiments.}
\begin{tabular}{lccccc}
\toprule
dataset     & \# (training) & \# (calibration) & \# (test) & T & H \\
\midrule
{\tt chlorine} & 2584 & 861 & 861& 150 & 16 \\
{\tt covid}       & 228  & 76 & 76 & 100    & 10   \\
{\tt ERA5} & 1228 & 410 & 410    & 358    & 7      \\
{\tt ecg} & 3000 & 1000 & 1000 & 130 & 10 \\
{\tt eeg}    & 23040 & 7680 & 7680 & 40     & 10      \\
{\tt faceall} & 1350 & 450 & 450 & 111 & 20 \\
{\tt fiftywords} & 543 & 181 & 181 & 250 & 20 \\
{\tt ford} & 2952 & 984 & 984 & 450 & 50 \\
{\tt freezer} & 1800 & 600 & 600 & 271 & 30 \\
{\tt handoutlines} & 822 & 274 & 274 & 270 & 30 \\
{\tt inlineskate} & 390 & 130 & 130 & 200 & 20 \\
{\tt ITpower} & 657 & 219 & 219 & 18 & 6 \\
{\tt melbourne} & 1365 & 455 & 455 & 17 & 6 \\
{\tt middlephalanx} & 534 & 178 & 178 & 70 & 10 \\
{\tt mixedshapes} & 1755 & 585 & 585 & 1000 & 24 \\
{\tt fetalecg} & 2259 & 753 & 753 & 700 & 50 \\
{\tt refrigeration} & 450 & 150 & 150 & 700 & 20 \\
{\tt sonyaibrobot} & 372 & 124 & 124 & 50 & 15 \\
{\tt swedishleaf} & 675 & 225 & 225 & 100 & 28 \\
{\tt symbols} & 612 & 204 & 204 & 350 & 48 \\
{\tt temperature}  & 500 & 166  & 166   & 698    & 30      \\
{\tt wafer} & 4298 & 1432 & 1432 & 140 & 12 \\
{\tt wordrecognition} & 345 & 115 & 115 & 134 & 10 \\
{\tt yoga} & 1980 & 660 & 660 & 400 & 26 \\
\bottomrule
\end{tabular}
\label{tab:datasets}
\end{table*}
\paragraph{\texttt{eeg}~\citep{kelly2023uci}} The data is available at~\url{https://archive.ics.uci.edu/ml/datasets/EEG+Database}. We used the medium version of the dataset, aiming to forecast the responses of ten control subjects to visual stimuli of three types. We excluded alcoholic subjects from our experiments, as prior work and summary statistics indicate that EEG responses from control subjects are more challenging to predict~\citep{stankeviciute2021conformal}. Each subject had repeated trials for each type of stimulus, and each trial produced a time series for $64$ EEG channels, corresponding to $64$ sensors. We treated each individual trial and each channel as a separate time-series resulting in a total of $38400$ time-series. Following~\citet{stankeviciute2021conformal}, we downsampled the sequences (originally of length 255) to a length of 50, where we used the first $40$ time steps as input and the final $10$ as the prediction horizon. 
\paragraph{\texttt{covid}~\citep{ukhsa2025}} The data is available at~\url{https://coronavirus.data.gov.uk/}. We used regional-level data from the United Kingdom, divided into $380$ lower-tier local authorities, resulting in $380$ time series. We considered daily new COVID-19 cases over a $110$-day period, starting in mid-September $2020$ and ending in January $2021$. Each time series was split into an input sequence of $100$ days and a prediction horizon of $10$ days.

\paragraph{\texttt{temperature}~\citep{moutadid2023subseasonal}} The data was retrieved from~\url{https://github.com/microsoft/subseasonal_data/}. This dataset contains weather data for $832$ different locations in the United States. Our objective is to forecast the daily temperature computed as the average between the max and the min temperature for the years $2022-2023$. Each time series (\ie location) is split into an input sequence of $698$ days and a forecast horizon of $30$ days (\ie the final month). 

\paragraph{\texttt{ERA5}~\citep{nguyen2023climate}} The data is available at~\url{https://github.com/aditya-grover/climate-learn}. Since the full ERA5 dataset is too large for most deep learning models, ClimateLearn provides access to a reduced version derived from WeatherBench~\citep{rasp2020weatherbench}. This version regrids the raw data to a lower spatial resolution of $5.625^\circ$, resulting in a total of $2048$ distinct geographic locations (\ie time-series). Our objective is to forecast the daily maximum temperature for the year $2018$. Each time series is split into an input sequence of $358$ days and a prediction horizon of $7$ days (\ie the final week). 

\paragraph{UCR repository~\citep{dau2019ucr}} The remaining datasets were retrieved from the UCR repository. We removed the original target variable and treat each dataset as a forecasting task, with different forecasting horizons ranging from $6$ to $50$ steps. 

\section{Experiments: extended results}
\subsection{Q1: detailed results}\label{sec:Q1app}
For completeness, we provide the detailed tabular results corresponding to the experiments in Q1. \Tabref{tab:Q1app} reports the average empirical selective risk $\hat{\calR}$ ($\pm$ std) for all datasets and methods across six target coverage levels in the full abstention setting.

Additionally, \Figref{fig:Q1cd} shows the Critical Difference (CD) diagrams for the statistical tests conducted in Q1. These plots confirm that \firstmethod significantly outperforms the competitors, while the other baselines perform similarly across all considered target coverage levels.
\begin{longtable}{l|c|c|ccc}
\caption{Each dataset's average empirical selective risk ($\pm$ std) for all methods across six target coverages in the full abstention setting. \firstmethod achieves a lower selective risk than the baselines in most of the dataset-coverage pairs.}
  \label{tab:Q1app}\\
\toprule
  dataset & c & \firstmethod & \AdaptiveCF & \DPRNN & \MQRNN \\
\cline{1-6}
\endfirsthead
\toprule
dataset & c & \firstmethod & \AdaptiveCF & \DPRNN & \MQRNN \\
\cline{1-6}
\endhead 
{\tt chlorine} & .70 & \textbf{.0048 $\pm$ .0003} & .0057 $\pm$ .0006 & .0066 $\pm$ .0006 & .0056 $\pm$ .0006 \\*
 & .75 & \textbf{.0049 $\pm$ .0003} & .0058 $\pm$ .0006 & .0065 $\pm$ .0006 & .0057 $\pm$ .0006 \\*
 & .80 & \textbf{.0051 $\pm$ .0003} & .0059 $\pm$ .0006 & .0066 $\pm$ .0006 & .0059 $\pm$ .0006 \\*
 & .85 & \textbf{.0053 $\pm$ .0003} & .0060 $\pm$ .0006 & .0066 $\pm$ .0006 & .0060 $\pm$ .0005 \\*
 & .90 & \textbf{.0056 $\pm$ .0004} & .0062 $\pm$ .0006 & .0065 $\pm$ .0006 & .0062 $\pm$ .0005 \\*
 & .95 & \textbf{.0059 $\pm$ .0004} & .0064 $\pm$ .0006 & .0065 $\pm$ .0006 & .0064 $\pm$ .0005 \\*
\cline{1-6}
{\tt ERA5} & .70 & \textbf{.0044 $\pm$ .0006} & .0062 $\pm$ .0014 & .0136 $\pm$ .0019 & .0067 $\pm$ .0014 \\*
 & .75 & \textbf{.0056 $\pm$ .0009} & .0071 $\pm$ .0012 & .0140 $\pm$ .0014 & .0076 $\pm$ .0012 \\*
 & .80 & \textbf{.0068 $\pm$ .0007} & .0085 $\pm$ .0011 & .0145 $\pm$ .0011 & .0087 $\pm$ .0011 \\*
 & .85 & \textbf{.0089 $\pm$ .0014} & .0099 $\pm$ .0013 & .0149 $\pm$ .0012 & .0103 $\pm$ .0012 \\*
 & .90 & \textbf{.0108 $\pm$ .0015} & .0120 $\pm$ .0013 & .0157 $\pm$ .0010 & .0122 $\pm$ .0019 \\*
 & .95 & \textbf{.0135 $\pm$ .0020} & .0145 $\pm$ .0018 & .0166 $\pm$ .0011 & .0146 $\pm$ .0013 \\*
\cline{1-6}
{\tt handoutlines} & .70 & \textbf{.0045 $\pm$ .0021} & .0055 $\pm$ .0023 & .0071 $\pm$ .0030 & .0110 $\pm$ .0089 \\*
 & .75 & \textbf{.0051 $\pm$ .0027} & .0054 $\pm$ .0021 & .0070 $\pm$ .0028 & .0122 $\pm$ .0116 \\*
 & .80 & \textbf{.0057 $\pm$ .0030} & .0059 $\pm$ .0028 & .0071 $\pm$ .0026 & .0120 $\pm$ .0108 \\*
 & .85 & \textbf{.0062 $\pm$ .0029} & .0063 $\pm$ .0027 & .0071 $\pm$ .0025 & .0121 $\pm$ .0100 \\*
 & .90 & \textbf{.0064 $\pm$ .0031} & \textbf{.0064 $\pm$ .0027} & .0072 $\pm$ .0026 & .0139 $\pm$ .0149 \\*
 & .95 & .0084 $\pm$ .0062 & .0101 $\pm$ .0062 & \textbf{.0078 $\pm$ .0023} & .0142 $\pm$ .0136 \\*
\cline{1-6}
{\tt inlineskate} & .70 & \textbf{.0031 $\pm$ .0005} & .0036 $\pm$ .0018 & .0036 $\pm$ .0004 & .0049 $\pm$ .0019 \\*
 & .75 & \textbf{.0032 $\pm$ .0006} & .0037 $\pm$ .0017 & .0036 $\pm$ .0004 & .0051 $\pm$ .0018 \\*
 & .80 & \textbf{.0034 $\pm$ .0005} & .0038 $\pm$ .0016 & .0039 $\pm$ .0002 & .0051 $\pm$ .0016 \\*
 & .85 & \textbf{.0038 $\pm$ .0005} & .0039 $\pm$ .0015 & \textbf{.0038 $\pm$ .0002} & .0052 $\pm$ .0016 \\*
 & .90 & \textbf{.0039 $\pm$ .0005} & .0042 $\pm$ .0014 & .0041 $\pm$ .0003 & .0051 $\pm$ .0015 \\*
 & .95 & \textbf{.0041 $\pm$ .0004} & .0046 $\pm$ .0013 & .0042 $\pm$ .0004 & .0052 $\pm$ .0014 \\*
\cline{1-6}
{\tt freezer} & .70 & \textbf{.0026 $\pm$ .0026} & .0128 $\pm$ .0183 & .0329 $\pm$ .0230 & .0038 $\pm$ .0039 \\*
 & .75 & \textbf{.0029 $\pm$ .0026} & .0131 $\pm$ .0180 & .0437 $\pm$ .0285 & .0045 $\pm$ .004 \\*
 & .80 & \textbf{.0106 $\pm$ .0139} & .0162 $\pm$ .0176 & .0542 $\pm$ .0367 & .0140 $\pm$ .0124 \\*
 & .85 & \textbf{.0170 $\pm$ .0141} & .0242 $\pm$ .0155 & .0620 $\pm$ .0400 & .0224 $\pm$ .0136 \\*
 & .90 & \textbf{.0238 $\pm$ .0139} & .0320 $\pm$ .0174 & .0627 $\pm$ .0401 & .0307 $\pm$ .0188 \\*
 & .95 & \textbf{.0340 $\pm$ .0154} & .0440 $\pm$ .0294 & .0622 $\pm$ .0410 & .0427 $\pm$ .0305 \\*
\cline{1-6}
{\tt wafer} & .70 & \textbf{.0002 $\pm$ .0001} & .0023 $\pm$ .0035 & .0097 $\pm$ .0030 & .0009 $\pm$ .0005 \\*
 & .75 & \textbf{.0002 $\pm$ .0001} & .0025 $\pm$ .0036 & .0106 $\pm$ .0035 & .0010 $\pm$ .0004 \\*
 & .80 & \textbf{.0007 $\pm$ .0014} & .0027 $\pm$ .0035 & .0107 $\pm$ .0032 & .0010 $\pm$ .0004 \\*
 & .85 & \textbf{.0010 $\pm$ .0014} & .0029 $\pm$ .0034 & .0111 $\pm$ .0034 & .0011 $\pm$ .0004 \\*
 & .90 & \textbf{.0012 $\pm$ .0013} & .0031 $\pm$ .0033 & .0109 $\pm$ .0031 & .0015 $\pm$ .0006 \\*
 & .95 & \textbf{.0014 $\pm$ .0013} & .0035 $\pm$ .0033 & .0109 $\pm$ .0027 & .0020 $\pm$ .0005 \\*
\cline{1-6}
{\tt ecg} & .70 & \textbf{.0593 $\pm$ .0043} & .0819 $\pm$ .0130 & .1317 $\pm$ .0166 & .0789 $\pm$ .0103 \\*
 & .75 & \textbf{.0631 $\pm$ .0045} & .0867 $\pm$ .0121 & .1287 $\pm$ .0163 & .0827 $\pm$ .0107 \\*
 & .80 & \textbf{.0683 $\pm$ .0065} & .0923 $\pm$ .0140 & .1283 $\pm$ .0166 & .0874 $\pm$ .0117 \\*
 & .85 & \textbf{.0736 $\pm$ .0072} & .0982 $\pm$ .0131 & .1274 $\pm$ .0157 & .0934 $\pm$ .0122 \\*
 & .90 & \textbf{.0839 $\pm$ .0072} & .1038 $\pm$ .0139 & .1261 $\pm$ .0149 & .0992 $\pm$ .0131 \\*
 & .95 & \textbf{.0957 $\pm$ .0072} & .1116 $\pm$ .0126 & .1245 $\pm$ .0146 & .1076 $\pm$ .0133 \\*
\cline{1-6}
{\tt ITpower} & .70 & \textbf{.0669 $\pm$ .0104} & .0913 $\pm$ .0209 & .0874 $\pm$ .0147 & .0927 $\pm$ .0192 \\*
 & .75 & \textbf{.0770 $\pm$ .0126} & .0991 $\pm$ .0209 & .0898 $\pm$ .0152 & .0971 $\pm$ .0191 \\*
 & .80 & \textbf{.0826 $\pm$ .0109} & .1034 $\pm$ .0202 & .0904 $\pm$ .0129 & .0984 $\pm$ .0178 \\*
 & .85 & \textbf{.0883 $\pm$ .0125} & .1078 $\pm$ .0185 & .0931 $\pm$ .0132 & .1028 $\pm$ .0190 \\*
 & .90 & \textbf{.0972 $\pm$ .0097} & .1113 $\pm$ .0179 & .1017 $\pm$ .0136 & .1063 $\pm$ .0178 \\*
 & .95 & \textbf{.1079 $\pm$ .0126} & .1172 $\pm$ .0158 & .1104 $\pm$ .0119 & .1112 $\pm$ .0149 \\*
\cline{1-6}
{\tt melbourne} & .70 & \textbf{.0336 $\pm$ .0074} & .0866 $\pm$ .0861 & .0755 $\pm$ .0214 & .0637 $\pm$ .0264 \\*
 & .75 & \textbf{.0437 $\pm$ .0135} & .0895 $\pm$ .0821 & .0743 $\pm$ .0199 & .0670 $\pm$ .0246 \\*
 & .80 & \textbf{.0542 $\pm$ .0188} & .0920 $\pm$ .0782 & .0731 $\pm$ .0189 & .0709 $\pm$ .0236 \\*
 & .85 & \textbf{.0598 $\pm$ .0173} & .1014 $\pm$ .0801 & .0739 $\pm$ .0199 & .0754 $\pm$ .0238 \\*
 & .90 & .0901 $\pm$ .0599 & .1087 $\pm$ .0744 & \textbf{.0766 $\pm$ .0203} & .0855 $\pm$ .0328 \\*
 & .95 & .1070 $\pm$ .0607 & .1182 $\pm$ .0713 & \textbf{.0909 $\pm$ .0193} & .0977 $\pm$ .0300 \\*
\cline{1-6}
{\tt mixedshapes} & .70 & \textbf{.0291 $\pm$ .0033} & .0397 $\pm$ .0066 & .0542 $\pm$ .0058 & .0421 $\pm$ .0084 \\*
 & .75 & \textbf{.0322 $\pm$ .0034} & .0425 $\pm$ .0058 & .0548 $\pm$ .0057 & .0453 $\pm$ .0085 \\*
 & .80 & \textbf{.0358 $\pm$ .0028} & .0454 $\pm$ .0066 & .0554 $\pm$ .0051 & .0478 $\pm$ .0084 \\*
 & .85 & \textbf{.0393 $\pm$ .0027} & .0486 $\pm$ .0059 & .0562 $\pm$ .0056 & .0526 $\pm$ .0075 \\*
 & .90 & \textbf{.0435 $\pm$ .0025} & .0541 $\pm$ .0063 & .0581 $\pm$ .0062 & .0561 $\pm$ .0071 \\*
 & .95 & \textbf{.0506 $\pm$ .0042} & .0579 $\pm$ .0057 & .0605 $\pm$ .0053 & .0595 $\pm$ .0068 \\*
\cline{1-6}
{\tt symbols} & .70 & \textbf{.0263 $\pm$ .0133} & .0352 $\pm$ .0092 & .0464 $\pm$ .0157 & .0453 $\pm$ .0115 \\*
 & .75 & \textbf{.0260 $\pm$ .0124} & .0378 $\pm$ .0104 & .0465 $\pm$ .0137 & .0465 $\pm$ .0124 \\*
 & .80 & \textbf{.0273 $\pm$ .0113} & .0391 $\pm$ .0094 & .0472 $\pm$ .0122 & .0457 $\pm$ .0115 \\*
 & .85 & \textbf{.0295 $\pm$ .0098} & .0407 $\pm$ .0086 & .0461 $\pm$ .0109 & .0460 $\pm$ .0098 \\*
 & .90 & \textbf{.0311 $\pm$ .0092} & .0430 $\pm$ .0081 & .0469 $\pm$ .0108 & .0463 $\pm$ .0095 \\*
 & .95 & \textbf{.0344 $\pm$ .0093} & .0466 $\pm$ .0097 & .0471 $\pm$ .0100 & .0478 $\pm$ .0082 \\*
\cline{1-6}
{\tt yoga} & .70 & \textbf{.0194 $\pm$ .0039} & .0223 $\pm$ .0020 & .0322 $\pm$ .0055 & .0244 $\pm$ .0049 \\*
 & .75 & \textbf{.0202 $\pm$ .0041} & .0231 $\pm$ .0022 & .0326 $\pm$ .0054 & .0249 $\pm$ .0046 \\*
 & .80 & \textbf{.0215 $\pm$ .0042} & .0244 $\pm$ .0018 & .0329 $\pm$ .0050 & .0260 $\pm$ .0044 \\*
 & .85 & \textbf{.0227 $\pm$ .0042} & .0257 $\pm$ .0019 & .0337 $\pm$ .0045 & .0269 $\pm$ .0046 \\*
 & .90 & \textbf{.0243 $\pm$ .0045} & .0280 $\pm$ .0048 & .0342 $\pm$ .0042 & .0279 $\pm$ .0046 \\*
 & .95 & \textbf{.0255 $\pm$ .0048} & .0297 $\pm$ .0045 & .0343 $\pm$ .0042 & .0293 $\pm$ .0042 \\*
\cline{1-6}
{\tt covid} & .70 & \textbf{.2950 $\pm$ .2312} & .3029 $\pm$ .0946 & .2979 $\pm$ .1309 & .4046 $\pm$ .1386 \\*
 & .75 & .2959 $\pm$ .1931 & .3081 $\pm$ .0888 & \textbf{.2893 $\pm$ .1275} & .4170 $\pm$ .1338 \\*
 & .80 & .2970 $\pm$ .1755 & .3124 $\pm$ .0890 & \textbf{.2910 $\pm$ .1225} & .4110 $\pm$ .1297 \\*
 & .85 & .3002 $\pm$ .1653 & .3175 $\pm$ .0847 & \textbf{.2896 $\pm$ .1146} & .4260 $\pm$ .1218 \\*
 & .90 & .3162 $\pm$ .1599 & .3271 $\pm$ .0817 & \textbf{.3032 $\pm$ .1111} & .4155 $\pm$ .1130 \\*
 & .95 & .3464 $\pm$ .1417 & .3381 $\pm$ .0721 & \textbf{.3212 $\pm$ .1031} & .4175 $\pm$ .1112 \\*
\cline{1-6}
{\tt eeg} & .70 & \textbf{.1445 $\pm$ .0044} & .1623 $\pm$ .0147 & .2235 $\pm$ .0316 & .1669 $\pm$ .0143 \\*
 & .75 & \textbf{.1554 $\pm$ .0049} & .1753 $\pm$ .0179 & .2303 $\pm$ .0344 & .1806 $\pm$ .0155 \\*
 & .80 & \textbf{.1677 $\pm$ .0054} & .1929 $\pm$ .0231 & .2379 $\pm$ .0350 & .1955 $\pm$ .0171 \\*
 & .85 & \textbf{.1810 $\pm$ .0053} & .2109 $\pm$ .0256 & .2436 $\pm$ .0357 & .2126 $\pm$ .0217 \\*
 & .90 & \textbf{.1974 $\pm$ .0061} & .2352 $\pm$ .0301 & .2538 $\pm$ .0369 & .2322 $\pm$ .0279 \\*
 & .95 & \textbf{.2218 $\pm$ .0065} & .2694 $\pm$ .0382 & .2719 $\pm$ .0379 & .2588 $\pm$ .0298 \\*
\cline{1-6}
{\tt fiftywords} & .70 & \textbf{.2781 $\pm$ .1311} & .3999 $\pm$ .0639 & .3244 $\pm$ .0803 & .4310 $\pm$ .1196 \\*
 & .75 & \textbf{.2822 $\pm$ .1250} & .4152 $\pm$ .0633 & .3698 $\pm$ .1070 & .4430 $\pm$ .1339 \\*
 & .80 & \textbf{.3025 $\pm$ .1158} & .4299 $\pm$ .0756 & .3908 $\pm$ .1085 & .4727 $\pm$ .1185 \\*
 & .85 & \textbf{.3358 $\pm$ .1096} & .4643 $\pm$ .0898 & .4033 $\pm$ .1133 & .4805 $\pm$ .1199 \\*
 & .90 & \textbf{.3641 $\pm$ .1081} & .5195 $\pm$ .0811 & .4395 $\pm$ .1241 & .5056 $\pm$ .1231 \\*
 & .95 & \textbf{.4168 $\pm$ .0917} & .5681 $\pm$ .0870 & .5307 $\pm$ .0972 & .5457 $\pm$ .1122 \\*
\cline{1-6}
{\tt wordrecognition} & .70 & \textbf{.1366 $\pm$ .0266} & .1572 $\pm$ .0245 & .1951 $\pm$ .0167 & .1645 $\pm$ .0297  \\*
 & .75 & \textbf{.1464 $\pm$ .0232} & .1599 $\pm$ .0263 & .1981 $\pm$ .0180 & .1709 $\pm$ .0252  \\*
 & .80 & \textbf{.1524 $\pm$ .0235} & .1638 $\pm$ .0234 & .2023 $\pm$ .0225 & .1733 $\pm$ .0264  \\*
 & .85 & \textbf{.1581 $\pm$ .0236} & .1690 $\pm$ .0222 & .1997 $\pm$ .0228 & .1785 $\pm$ .0243  \\*
 & .90 & \textbf{.1641 $\pm$ .0247} & .1723 $\pm$ .0187 & .1979 $\pm$ .0192 & .1850 $\pm$ .0198  \\*
 & .95 & \textbf{.1700 $\pm$ .0228} & .1859 $\pm$ .0175 & .1960 $\pm$ .0206 & .1886 $\pm$ .0183  \\*
\cline{1-6}
{\tt middlephalanx} & .70 & \textbf{.1180 $\pm$ .0320} & .1425 $\pm$ .0330 & .1742 $\pm$ .0294 & .1604 $\pm$ .0323 \\*
 & .75 & \textbf{.1233 $\pm$ .0307} & .1473 $\pm$ .0291 & .1664 $\pm$ .0274 & .1672 $\pm$ .0305 \\*
 & .80 & \textbf{.1304 $\pm$ .0286} & .1600 $\pm$ .0268 & .1610 $\pm$ .0252 & .1749 $\pm$ .0281 \\*
 & .85 & \textbf{.1405 $\pm$ .0274} & .1658 $\pm$ .0286 & .1595 $\pm$ .0236 & .1790 $\pm$ .0284 \\*
 & .90 & \textbf{.1505 $\pm$ .0276} & .1701 $\pm$ .0255 & .1606 $\pm$ .0220 & .1847 $\pm$ .0259 \\*
 & .95 & .1654 $\pm$ .0240 & .1799 $\pm$ .0207 & \textbf{.1644 $\pm$ .0197} & .1931 $\pm$ .0245 \\*
\cline{1-6}
{\tt swedishleaf} & .70 & .1735 $\pm$ .0432 & .1658 $\pm$ .0380 & .2684 $\pm$ .0393 & \textbf{.1619 $\pm$ .0154} \\*
 & .75 & .2176 $\pm$ .0461 & .2012 $\pm$ .0381 & .2992 $\pm$ .0484 & \textbf{.1988 $\pm$ .0311} \\*
 & .80 & .2664 $\pm$ .0490 & .2488 $\pm$ .0441 & .3192 $\pm$ .0479 & \textbf{.2281 $\pm$ .0285} \\*
 & .85 & .3164 $\pm$ .0440 & .2826 $\pm$ .0431 & .3455 $\pm$ .0584 & \textbf{.2754 $\pm$ .0341} \\*
 & .90 & .3641 $\pm$ .0462 & .3397 $\pm$ .0540 & .3752 $\pm$ .0554 & \textbf{.3185 $\pm$ .0423} \\*
 & .95 & .4369 $\pm$ .0369 & .4014 $\pm$ .0560 & .4115 $\pm$ .0511 & \textbf{.3750 $\pm$ .0431} \\*
\cline{1-6}
{\tt faceall} & .70 & \textbf{.3953 $\pm$ .0438} & .4719 $\pm$ .0732 & .4972 $\pm$ .0259 & .4738 $\pm$ .0562 \\*
 & .75 & \textbf{.4160 $\pm$ .0408} & .4924 $\pm$ .0725 & .5124 $\pm$ .0285 & .4974 $\pm$ .0562 \\*
 & .80 & \textbf{.4333 $\pm$ .0395} & .5205 $\pm$ .0718 & .5349 $\pm$ .0250 & .5223 $\pm$ .0530 \\*
 & .85 & \textbf{.4588 $\pm$ .0393} & .5580 $\pm$ .0905 & .5618 $\pm$ .0298 & .5434 $\pm$ .0516 \\*
 & .90 & \textbf{.4865 $\pm$ .0371} & .5870 $\pm$ .0830 & .5721 $\pm$ .0356 & .5658 $\pm$ .0542 \\*
 & .95 & \textbf{.5446 $\pm$ .0466} & .6583 $\pm$ .1043 & .6039 $\pm$ .0299 & .6051 $\pm$ .0473 \\*
\cline{1-6}
{\tt fetalecg} & .70 & \textbf{.3919 $\pm$ .0863} & .4755 $\pm$ .1081 & .4922 $\pm$ .0938 & .4978 $\pm$ .1245 \\*
 & .75 & \textbf{.4058 $\pm$ .0956} & .4945 $\pm$ .1115 & .5003 $\pm$ .0953 & .5058 $\pm$ .1239 \\*
 & .80 & \textbf{.4141 $\pm$ .1007} & .5081 $\pm$ .1156 & .5080 $\pm$ .0972 & .5150 $\pm$ .125 \\*
 & .85 & \textbf{.4204 $\pm$ .0996} & .5223 $\pm$ .1176 & .5195 $\pm$ .1006 & .5218 $\pm$ .1235 \\*
 & .90 & \textbf{.4317 $\pm$ .0983} & .5351 $\pm$ .1149 & .5329 $\pm$ .1052 & .5285 $\pm$ .1225 \\*
 & .95 & \textbf{.4414 $\pm$ .0963} & .5498 $\pm$ .1193 & .5420 $\pm$ .1071 & .5388 $\pm$ .1194 \\*
\cline{1-6}
{\tt ford} & .70 & .4816 $\pm$ .0162 & .4898 $\pm$ .0247 & .5745 $\pm$ .0155 & \textbf{.4708 $\pm$ .0368} \\*
 & .75 & .4958 $\pm$ .0159 & .5041 $\pm$ .0215 & .5746 $\pm$ .0140 & \textbf{.4882 $\pm$ .0336} \\*
 & .80 & .5109 $\pm$ .0156 & .5182 $\pm$ .0179 & .5744 $\pm$ .0142 & \textbf{.5058 $\pm$ .0320} \\*
 & .85 & .5229 $\pm$ .0175 & .5364 $\pm$ .0151 & .5749 $\pm$ .0140 & \textbf{.5225 $\pm$ .0297} \\*
 & .90 & \textbf{.5338 $\pm$ .0168} & .5504 $\pm$ .0143 & .5731 $\pm$ .0147 & .5388 $\pm$ .0257 \\*
 & .95 & \textbf{.5482 $\pm$ .0186} & .5624 $\pm$ .0154 & .5733 $\pm$ .0136 & .5563 $\pm$ .0205 \\*
\cline{1-6}
{\tt refrigeration} & .70 & \textbf{.5739 $\pm$ .0763} & .7427 $\pm$ .1127 & .8270 $\pm$ .0889 & .7268 $\pm$ .1109 \\*
 & .75 & \textbf{.5896 $\pm$ .0732} & .7472 $\pm$ .1151 & .8178 $\pm$ .0896 & .7411 $\pm$ .1060 \\*
 & .80 & \textbf{.6071 $\pm$ .0840} & .7615 $\pm$ .1128 & .8113 $\pm$ .0907 & .7512 $\pm$ .0989 \\*
 & .85 & \textbf{.6181 $\pm$ .0829} & .7826 $\pm$ .1052 & .8076 $\pm$ .0898 & .7620 $\pm$ .0972 \\*
 & .90 & \textbf{.6319 $\pm$ .0746} & .7946 $\pm$ .1104 & .8131 $\pm$ .0943 & .7715 $\pm$ .1026 \\*
 & .95 & \textbf{.6607 $\pm$ .0778} & .8104 $\pm$ .1071 & .8151 $\pm$ .1031 & .8099 $\pm$ .1133 \\*
\cline{1-6}
{\tt sonyaiborobot} & .70 & \textbf{.4444 $\pm$ .0358} & .4599 $\pm$ .0365 & .4822 $\pm$ .0609 & .4699 $\pm$ .0433 \\*
 & .75 & \textbf{.4633 $\pm$ .0405} & .4958 $\pm$ .0305 & .5030 $\pm$ .0586 & .4933 $\pm$ .0462 \\*
 & .80 & \textbf{.4905 $\pm$ .0381} & .5168 $\pm$ .0317 & .5160 $\pm$ .0526 & .5208 $\pm$ .0453 \\*
 & .85 & \textbf{.5078 $\pm$ .0345} & .5567 $\pm$ .0496 & .5368 $\pm$ .0551 & .5395 $\pm$ .0463 \\*
 & .90 & \textbf{.5329 $\pm$ .0368} & .5779 $\pm$ .0551 & .5645 $\pm$ .0503 & .5575 $\pm$ .0457 \\*
 & .95 & \textbf{.5590 $\pm$ .0331} & .5918 $\pm$ .0505 & .5839 $\pm$ .0504 & .5815 $\pm$ .0524 \\*
\cline{1-6}
{\tt temperature} & .70 & \textbf{.6102 $\pm$ .1080} & .7435 $\pm$ .2824 & .9294 $\pm$ .4112 & .7884 $\pm$ .3538 \\*
 & .75 & \textbf{.6666 $\pm$ .1562} & .7617 $\pm$ .2836 & .9249 $\pm$ .4060 & .8035 $\pm$ .3532 \\*
 & .80 & \textbf{.6923 $\pm$ .1644} & .8074 $\pm$ .2969 & .9249 $\pm$ .4017 & .8273 $\pm$ .3476 \\*
 & .85 & \textbf{.7259 $\pm$ .1787} & .8160 $\pm$ .2989 & .9233 $\pm$ .3949 & .8340 $\pm$ .3524 \\*
 & .90 & \textbf{.7757 $\pm$ .1708} & .8423 $\pm$ .3285 & .9145 $\pm$ .3885 & .8644 $\pm$ .3659 \\*
 & .95 & \textbf{.8270 $\pm$ .1648} & .8658 $\pm$ .3315 & .9037 $\pm$ .3784 & .8837 $\pm$ .3703 \\*
\bottomrule
\end{longtable}

\begin{figure}[h!]
    \centering
    \includegraphics[width=1\linewidth]{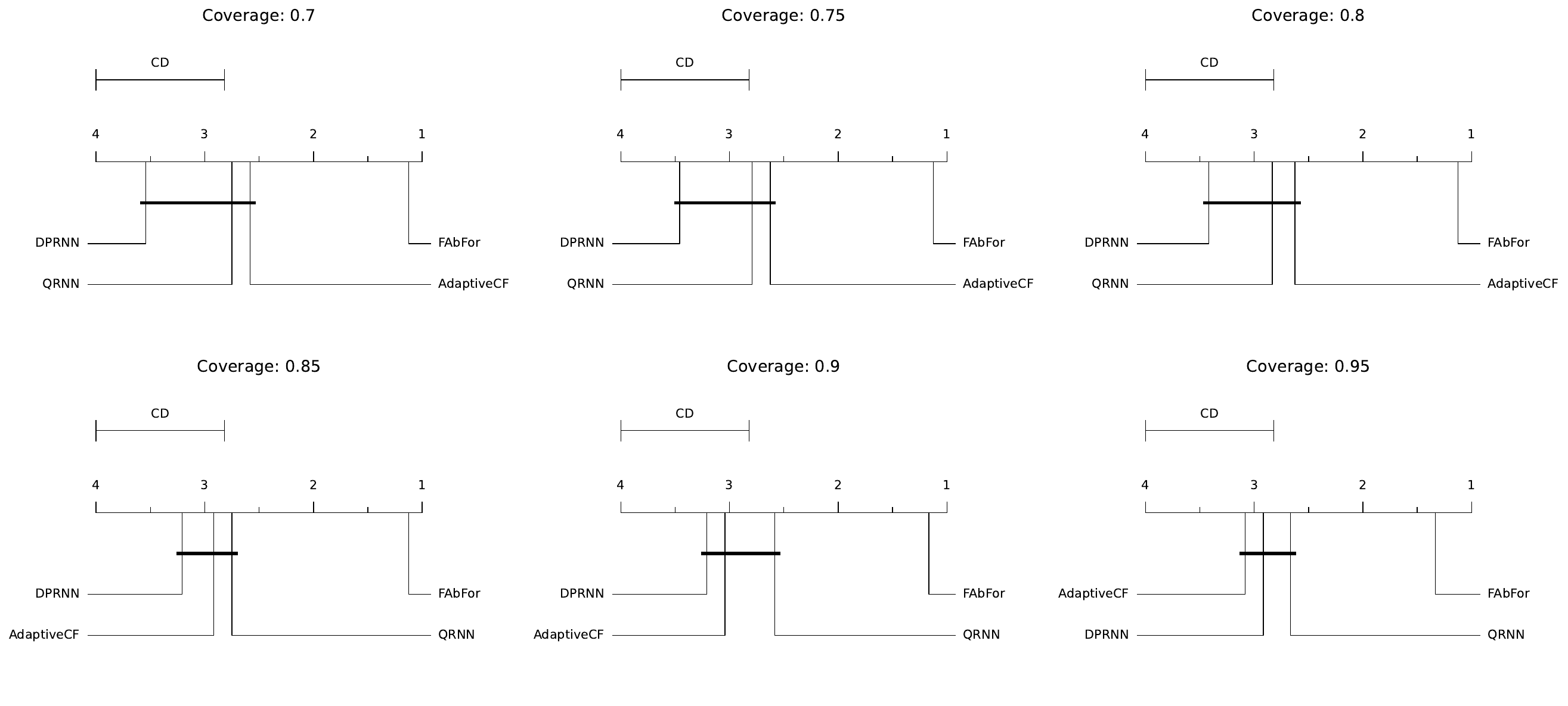}
    \caption{CD plots of empirical selective risk for different target coverages. \firstmethod significantly outperforms the baselines for all considered target coverages.}
    \label{fig:Q1cd}
\end{figure}

\subsection{Q2 : detailed results}\label{sec:Q2app}
For completeness, \Figref{fig:Q2app} reports each dataset's average selective risk $\calR$ for our three proposed strategies and \dummy across six target coverage levels. We complement this graphical representation with~\Tabref{tab:Q2app}, which provides the detailed numerical values, reporting each dataset's average empirical selective risk ($\pm$ std) for the considered methods.

\begin{longtable}{l|c|ccc|c}
\caption{Each dataset's average empirical selective risk ($\pm$ std) for all methods across six target coverages in the full abstention setting. \firstmethod achieves a lower selective risk than the baselines in most of the dataset-coverage pairs.}
  \label{tab:Q2app}\\
\toprule
  dataset & c & \firstmethod & \secondmethod & \thirdmethod & \dummy \\*
\cline{1-6}
\endfirsthead
\toprule
  dataset & c & \firstmethod & \secondmethod & \thirdmethod & \dummy \\*
\cline{1-6}
\endhead 
{\tt chlorine} & .70 & .0048 $\pm$ .0003 & .0038 $\pm$ .0002 & \textbf{.0034 $\pm$ .0003} & .0105 $\pm$ .0034 \\*
 & .75 & .0049 $\pm$ .0003 & .0040 $\pm$ .0002 & \textbf{.0036 $\pm$ .0003} & .0103 $\pm$ .0032 \\*
 & .80 & .0051 $\pm$ .0003 & .0042 $\pm$ .0003 & \textbf{.0039 $\pm$ .0003} & .0101 $\pm$ .0030 \\*
 & .85 & .0053 $\pm$ .0003 & .0045 $\pm$ .0002 & \textbf{.0042 $\pm$ .0003} & .0100 $\pm$ .0028 \\*
 & .90 & .0056 $\pm$ .0004 & .0049 $\pm$ .0003 & \textbf{.0046 $\pm$ .0004} & .0100 $\pm$ .0026 \\*
 & .95 & .0059 $\pm$ .0004 & .0056 $\pm$ .0003 & \textbf{.0052 $\pm$ .0004} & .0101 $\pm$ .0025 \\*
\cline{1-6}
{\tt ERA5} & .70 & .0044 $\pm$ .0006 & \textbf{.0043 $\pm$ .0005} & \textbf{.0043 $\pm$ .0006} & .0160 $\pm$ .0015 \\*
 & .75 & .0056 $\pm$ .0009 & \textbf{.0055 $\pm$ .0007} & \textbf{.0055 $\pm$ .0007} & .0163 $\pm$ .0016 \\*
 & .80 & .0068 $\pm$ .0007 & \textbf{.0067 $\pm$ .0008} & .0069 $\pm$ .0008 & .0164 $\pm$ .0015 \\*
 & .85 & .0089 $\pm$ .0014 & \textbf{.0085 $\pm$ .0011} & .0086 $\pm$ .0011 & .0166 $\pm$ .0015 \\*
 & .90 & .0108 $\pm$ .0015 & \textbf{.0106 $\pm$ .0014} & .0107 $\pm$ .0014 & .0168 $\pm$ .0015 \\*
 & .95 & .0135 $\pm$ .0020 & .0134 $\pm$ .0019 & \textbf{.0133 $\pm$ .0019} & .0169 $\pm$ .0016 \\*
\cline{1-6}
{\tt handoutlines} & .70 & .0045 $\pm$ .0021 & \textbf{.0041 $\pm$ .0020} & .0044 $\pm$ .0023 & .0140 $\pm$ .0075 \\*
 & .75 & .0051 $\pm$ .0027 & \textbf{.0046 $\pm$ .0023} & .0050 $\pm$ .0026 & .0154 $\pm$ .0083 \\*
 & .80 & .0057 $\pm$ .0030 & \textbf{.0053 $\pm$ .0025} & \textbf{.0053 $\pm$ .0027} & .0172 $\pm$ .0095 \\*
 & .85 & .0062 $\pm$ .0029 & \textbf{.0057 $\pm$ .0025} & .0058 $\pm$ .0031 & .0187 $\pm$ .0104 \\*
 & .90 & .0064 $\pm$ .0031 & \textbf{.0063 $\pm$ .0030} & .0064 $\pm$ .0035 & .0207 $\pm$ .0117 \\*
 & .95 & .0084 $\pm$ .0062 & .0086 $\pm$ .0057 & \textbf{.0080 $\pm$ .0056} & .0227 $\pm$ .0129 \\*
\cline{1-6}
{\tt freezer} & .70 & .0026 $\pm$ .0026 & \textbf{.0023 $\pm$ .0024} & .0057 $\pm$ .0053 & .0592 $\pm$ .0225 \\*
 & .75 & .0029 $\pm$ .0026 & \textbf{.0022 $\pm$ .0022} & .0055 $\pm$ .0045 & .0599 $\pm$ .0224 \\*
 & .80 & .0106 $\pm$ .0139 & \textbf{.0023 $\pm$ .0021} & .0062 $\pm$ .0051 & .0610 $\pm$ .0211 \\*
 & .85 & .0170 $\pm$ .0141 & .0105 $\pm$ .0133 & \textbf{.0065 $\pm$ .0049} & .0618 $\pm$ .0203 \\*
 & .90 & .0238 $\pm$ .0139 & .0150 $\pm$ .0131 & \textbf{.0113 $\pm$ .0131} & .0624 $\pm$ .0202 \\*
 & .95 & .0340 $\pm$ .0154 & \textbf{.0271 $\pm$ .0160} & .0282 $\pm$ .0183 & .0630 $\pm$ .0198 \\*
\cline{1-6}
{\tt wafer} & .70 & \textbf{.0002 $\pm$ .0001} & .0003 $\pm$ .0003 & .0003 $\pm$ .0005 & .0133 $\pm$ .0045 \\*
 & .75 & \textbf{.0002 $\pm$ .0001} & .0004 $\pm$ .0007 & .0005 $\pm$ .0011 & .0133 $\pm$ .0045 \\*
 & .80 & .0007 $\pm$ .0014 & \textbf{.0004 $\pm$ .0007} & .0005 $\pm$ .0010 & .0133 $\pm$ .0047 \\*
 & .85 & .0010 $\pm$ .0014 & \textbf{.0005 $\pm$ .0009} & .0006 $\pm$ .0012 & .0132 $\pm$ .0045 \\*
 & .90 & .0012 $\pm$ .0013 & .0010 $\pm$ .0012 & \textbf{.0009 $\pm$ .0013} & .0132 $\pm$ .0047 \\*
 & .95 & .0014 $\pm$ .0013 & \textbf{.0013 $\pm$ .0013} & .0015 $\pm$ .0019 & .0133 $\pm$ .0048 \\*
\cline{1-6}
{\tt ecg} & .70 & .0593 $\pm$ .0043 & .0423 $\pm$ .0043 & \textbf{.0419 $\pm$ .0044} & .0599 $\pm$ .0051 \\*
 & .75 & .0631 $\pm$ .0045 & .0475 $\pm$ .0036 & \textbf{.0467 $\pm$ .0044} & .0650 $\pm$ .0055 \\*
 & .80 & .0683 $\pm$ .0065 & .0534 $\pm$ .0028 & \textbf{.0527 $\pm$ .0038} & .0703 $\pm$ .0061 \\*
 & .85 & .0736 $\pm$ .0072 & .0607 $\pm$ .0030 & \textbf{.0603 $\pm$ .0039} & .0772 $\pm$ .0069 \\*
 & .90 & .0839 $\pm$ .0072 & .0703 $\pm$ .0029 & \textbf{.0699 $\pm$ .0050} & .0831 $\pm$ .0072 \\*
 & .95 & .0957 $\pm$ .0072 & .0843 $\pm$ .0058 & \textbf{.0842 $\pm$ .0061} & .0978 $\pm$ .0080 \\*
\cline{1-6}
{\tt ITpower} & .70 & .0669 $\pm$ .0104 & .0689 $\pm$ .0081 & \textbf{.0639 $\pm$ .0065} & .1372 $\pm$ .0172 \\*
 & .75 & .0770 $\pm$ .0126 & .0772 $\pm$ .0098 & \textbf{.0680$\pm$ .0087} & .1340 $\pm$ .0166 \\*
 & .80 & .0826 $\pm$ .0109 & .0807 $\pm$ .0095 & \textbf{.0727 $\pm$ .0105} & .1305 $\pm$ .0163 \\*
 & .85 & .0883 $\pm$ .0125 & .0872 $\pm$ .0110 & \textbf{.0789 $\pm$ .0104} & .1284 $\pm$ .0155 \\*
 & .90 & .0972 $\pm$ .0097 & .0956 $\pm$ .0108 & \textbf{.0872 $\pm$ .0129} & .1253 $\pm$ .0155 \\*
 & .95 & .1079 $\pm$ .0126 & .1078 $\pm$ .0104 & \textbf{.0999 $\pm$ .0120} & .1233 $\pm$ .0146 \\*
\cline{1-6}
{\tt melbourne} & .70 & .0336 $\pm$ .0074 & .0307 $\pm$ .0061 & \textbf{.0305 $\pm$ .0058} & .1171 $\pm$ .0414 \\*
 & .75 & .0437 $\pm$ .0135 & \textbf{.0341 $\pm$ .0061} & .0363 $\pm$ .0091 & .1313 $\pm$ .0552 \\*
 & .80 & .0542 $\pm$ .0188 & \textbf{.0439 $\pm$ .0090} & .0440 $\pm$ .0092 & .1359 $\pm$ .0505 \\*
 & .85 & .0598 $\pm$ .0173 & \textbf{.0564 $\pm$ .0160} & \textbf{.0564 $\pm$ .0160} & .1436 $\pm$ .0542 \\*
 & .90 & .0901 $\pm$ .0599 & \textbf{.0675 $\pm$ .0154} & .0677 $\pm$ .0152 & .1490 $\pm$ .0541 \\*
 & .95 & .1070 $\pm$ .0607 & \textbf{.1010 $\pm$ .0576} & \textbf{.1010 $\pm$ .0576} & .1601 $\pm$ .0591 \\*
\cline{1-6}
{\tt mixedshapes} & .70 & .0291 $\pm$ .0033 & \textbf{.0198 $\pm$ .0032} & \textbf{.0198 $\pm$ .0026} & .0423 $\pm$ .0043 \\*
 & .75 & .0322 $\pm$ .0034 & \textbf{.0233 $\pm$ .0031} & \textbf{.0233 $\pm$ .0024} & .0464 $\pm$ .0047 \\*
 & .80 & .0358 $\pm$ .0028 & \textbf{.0275 $\pm$ .0028} & .0276 $\pm$ .0023 & .0505 $\pm$ .0051 \\*
 & .85 & .0393 $\pm$ .0027 & \textbf{.0326 $\pm$ .0026} & .0328 $\pm$ .0023 & .0543 $\pm$ .0056 \\*
 & .90 & .0435 $\pm$ .0025 & \textbf{.0388 $\pm$ .0023} & .0389 $\pm$ .0028 & .0586 $\pm$ .0063 \\*
 & .95 & .0506 $\pm$ .0042 & \textbf{.0467 $\pm$ .0031} & .0474 $\pm$ .0041 & .0622 $\pm$ .0063 \\*
\cline{1-6}
{\tt symbols} & .70 & .0263 $\pm$ .0133 & \textbf{.0233 $\pm$ .0105} & \textbf{.0233 $\pm$ .0105} & .0356 $\pm$ .0085 \\*
 & .75 & .0260 $\pm$ .0124 & .0245 $\pm$ .011 & \textbf{.0244 $\pm$ .0110} & .0371 $\pm$ .0088 \\*
 & .80 & .0273 $\pm$ .0113 & \textbf{.0263 $\pm$ .0108} & \textbf{.0263 $\pm$ .0108} & .0384 $\pm$ .0091 \\*
 & .85 & .0295 $\pm$ .0098 & \textbf{.0284 $\pm$ .0101} & \textbf{.0284 $\pm$ .0100} & .0394 $\pm$ .0092 \\*
 & .90 & .0311 $\pm$ .0092 & \textbf{.0308 $\pm$ .0096} & .0310 $\pm$ .0095 & .0406 $\pm$ .0094 \\*
 & .95 & \textbf{.0344 $\pm$ .0093} & .0359 $\pm$ .0095 & .0359 $\pm$ .0095 & .0416 $\pm$ .0096 \\*
\cline{1-6}
{\tt yoga} & .70 & .0194 $\pm$ .0039 & \textbf{.0144 $\pm$ .0034} & .0145 $\pm$ .0034 & .0205 $\pm$ .0039 \\*
 & .75 & .0202 $\pm$ .0041 & \textbf{.0155 $\pm$ .0035} & .0156 $\pm$ .0035 & .0221 $\pm$ .0040 \\*
 & .80 & .0215 $\pm$ .0042 & \textbf{.0169 $\pm$ .0035} & .0170 $\pm$ .0035 & .0239 $\pm$ .0043 \\*
 & .85 & .0227 $\pm$ .0042 & \textbf{.0187 $\pm$ .0036} & .0188 $\pm$ .0036 & .0257 $\pm$ .0046 \\*
 & .90 & .0243 $\pm$ .0045 & \textbf{.0210 $\pm$ .0040} & .0211 $\pm$ .0041 & .0276 $\pm$ .0050 \\*
 & .95 & .0255 $\pm$ .0048 & \textbf{.0240 $\pm$ .0046} & .0242 $\pm$ .0047 & .0300 $\pm$ .0005 \\*
\cline{1-6}
{\tt covid} & .70 & .2950 $\pm$ .2312 & .2927 $\pm$ .2304 & \textbf{.1974 $\pm$ .1894} & .3663 $\pm$ .1642 \\*
 & .75 & .2959 $\pm$ .1931 & .2900 $\pm$ .1879 & \textbf{.2075 $\pm$ .1765} & .3761 $\pm$ .1600 \\*
 & .80 & .2970 $\pm$ .1755 & .2929 $\pm$ .1823 & \textbf{.2199 $\pm$ .1670} & .3946 $\pm$ .1493 \\*
 & .85 & .3002 $\pm$ .1653 & .3058 $\pm$ .1651 & \textbf{.2405 $\pm$ .1612} & .4131 $\pm$ .1391 \\*
 & .90 & .3162 $\pm$ .1599 & .3176 $\pm$ .1586 & \textbf{.2896 $\pm$ .1397} & .4594 $\pm$ .1500 \\*
 & .95 & .3464 $\pm$ .1417 & \textbf{.3281 $\pm$ .1502} & .3462 $\pm$ .1207 & .4536 $\pm$ .1658 \\*
\cline{1-6}
{\tt eeg} & .70 & .1624 $\pm$ .0053 & \textbf{.1602 $\pm$ .0054} & .1603 $\pm$ .0056 & .3335 $\pm$ .0384 \\*
 & .75 & .1737 $\pm$ .0053 & \textbf{.1711 $\pm$ .0057} & \textbf{.1711 $\pm$ .0059} & .3398 $\pm$ .0389 \\*
 & .80 & .1861 $\pm$ .0057 & \textbf{.1837 $\pm$ .0055} & .1840 $\pm$ .0057 & .3436 $\pm$ .0381 \\*
 & .85 & .1994 $\pm$ .0067 & \textbf{.1966 $\pm$ .0066} & .1969 $\pm$ .0064 & .3481 $\pm$ .0367 \\*
 & .90 & .2134 $\pm$ .0076 & \textbf{.2118 $\pm$ .0067} & .2125 $\pm$ .0067 & .3513 $\pm$ .0358 \\*
 & .95 & .2379 $\pm$ .0104 & \textbf{.2336 $\pm$ .0098} & .2349 $\pm$ .0091 & .3516 $\pm$ .0336 \\*
\cline{1-6}
{\tt fiftywords} & .70 & .2781 $\pm$ .1311 & \textbf{.1865 $\pm$ .0623} & .2039 $\pm$ .0720 & .3784 $\pm$ .0799 \\*
 & .75 & .2822 $\pm$ .1250 & \textbf{.2153 $\pm$ .0724} & .2339 $\pm$ .0757 & .4018 $\pm$ .0859 \\*
 & .80 & .3025 $\pm$ .1158 & \textbf{.2572 $\pm$ .0913} & .2808 $\pm$ .0987 & .4245 $\pm$ .0940 \\*
 & .85 & .3358 $\pm$ .1096 & \textbf{.2980 $\pm$ .1048} & .3174 $\pm$ .1100 & .4460 $\pm$ .1011 \\*
 & .90 & .3641 $\pm$ .1081 & \textbf{.3388 $\pm$ .1011} & .3606 $\pm$ .1227 & .4666 $\pm$ .1041 \\*
 & .95 & .4168 $\pm$ .0917 & \textbf{.4098 $\pm$ .0968} & .4140 $\pm$ .1227 & .4852 $\pm$ .1051 \\*
\cline{1-6}
{\tt wordrecognition} & .70 & .1366 $\pm$ .0266 & \textbf{.1354 $\pm$ .0248} & .1448 $\pm$ .0233 & .1616 $\pm$ .0159 \\*
 & .75 & .1464 $\pm$ .0232 & \textbf{.1402 $\pm$ .0248} & .1505 $\pm$ .0230 & .1662 $\pm$ .0175 \\*
 & .80 & .1524 $\pm$ .0235 & \textbf{.1517 $\pm$ .0229} & .1564 $\pm$ .0233 & .1698 $\pm$ .0171 \\*
 & .85 & .1581 $\pm$ .0236 & \textbf{.1572 $\pm$ .0220} & .1609 $\pm$ .0244 & .1734 $\pm$ .0189 \\*
 & .90 & \textbf{.1641 $\pm$ .0247} & .1654 $\pm$ .0237 & .1658 $\pm$ .0241 & .1782 $\pm$ .0193 \\*
 & .95 & \textbf{.1700 $\pm$ .0228} & .1711 $\pm$ .0227 & .1715 $\pm$ .0232 & .1819 $\pm$ .0204 \\*
\cline{1-6}
{\tt middlephalanx} & .70 & .1180 $\pm$ .0320 & .1112 $\pm$ .0232 & \textbf{.1107 $\pm$ .0211} & .1950 $\pm$ .0256 \\*
 & .75 & .1233 $\pm$ .0307 & \textbf{.1205 $\pm$ .0205} & .1208 $\pm$ .0155 & .2013 $\pm$ .0266 \\*
 & .80 & .1304 $\pm$ .0286 & \textbf{.1273 $\pm$ .0195} & .1292 $\pm$ .0170 & .2024 $\pm$ .0261 \\*
 & .85 & .1405 $\pm$ .0274 & \textbf{.1354 $\pm$ .0206} & .1421 $\pm$ .0254 & .1999 $\pm$ .0256 \\*
 & .90 & .1505 $\pm$ .0276 & \textbf{.1467 $\pm$ .0258} & .1481 $\pm$ .0260 & .1969 $\pm$ .0255 \\*
 & .95 & .1654 $\pm$ .0240 & \textbf{.1580 $\pm$ .0252} & .1601 $\pm$ .0261 & .1918 $\pm$ .0257 \\*
\cline{1-6}
{\tt swedishleaf} & .70 & .1735 $\pm$ .0432 & .1694 $\pm$ .0393 & \textbf{.1465 $\pm$ .0232} & .4847 $\pm$ .0597 \\*
 & .75 & .2176 $\pm$ .0461 & .2045 $\pm$ .0407 & \textbf{.1815 $\pm$ .0289} & .4946 $\pm$ .0588 \\*
 & .80 & .2664 $\pm$ .0490 & .2425 $\pm$ .0507 & \textbf{.2228 $\pm$ .0249} & .5040 $\pm$ .0588 \\*
 & .85 & .3164 $\pm$ .0440 & .2972 $\pm$ .0433 & \textbf{.2602 $\pm$ .0267} & .5113 $\pm$ .0566 \\*
 & .90 & .3641 $\pm$ .0462 & .3571 $\pm$ .0461 & \textbf{.3277 $\pm$ .0315} & .5126 $\pm$ .0543 \\*
 & .95 & .4369 $\pm$ .0369 & .4346 $\pm$ .0417 & \textbf{.4097 $\pm$ .0453} & .5152 $\pm$ .0528 \\*
\cline{1-6}
{\tt faceall} & .70 & .3953 $\pm$ .0438 & .3868 $\pm$ .0438 & \textbf{.3855 $\pm$ .0416} & .7552 $\pm$ .1072 \\*
 & .75 & .4160 $\pm$ .0408 & \textbf{.4109 $\pm$ .0389} & .4137 $\pm$ .0429 & .7567 $\pm$ .1033 \\*
 & .80 & .4333 $\pm$ .0395 & \textbf{.4307 $\pm$ .0477} & .4372 $\pm$ .0490 & .7562 $\pm$ .0998 \\*
 & .85 & .4588 $\pm$ .0393 & \textbf{.4529 $\pm$ .0544} & .4633 $\pm$ .0541 & .7551 $\pm$ .0996 \\*
 & .90 & \textbf{.4865 $\pm$ .0371} & .4885 $\pm$ .0497 & .4972 $\pm$ .0495 & .7518 $\pm$ .0973 \\*
 & .95 & \textbf{.5446 $\pm$ .0466} & .5552 $\pm$ .0601 & .5689 $\pm$ .0650 & .7504 $\pm$ .0927 \\*
\cline{1-6}
{\tt fetalecg} & .70 & .3919 $\pm$ .0863 & \textbf{.3261 $\pm$ .0576} & .3262 $\pm$ .0575 & .4072 $\pm$ .0737 \\*
 & .75 & .4058 $\pm$ .0956 & \textbf{.3446 $\pm$ .0624} & .3447 $\pm$ .0619 & .4222 $\pm$ .0799 \\*
 & .80 & .4141 $\pm$ .1007 & .3665 $\pm$ .0681 & \textbf{.3643 $\pm$ .0669} & .4359 $\pm$ .0861 \\*
 & .85 & .4204 $\pm$ .0996 & .3879 $\pm$ .0776 & \textbf{.3870 $\pm$ .0773} & .4468 $\pm$ .0918 \\*
 & .90 & .4317 $\pm$ .0983 & .4116 $\pm$ .0822 & \textbf{.4103 $\pm$ .0831} & .4571 $\pm$ .0956 \\*
 & .95 & .4414 $\pm$ .0963 & .4415 $\pm$ .0973 & \textbf{.4357 $\pm$ .0926} & .4687 $\pm$ .0981 \\*
\cline{1-6}
{\tt ford} & .70 & .4816 $\pm$ .0162 & .4199 $\pm$ .0162 & \textbf{.4180 $\pm$ .0114} & .4639 $\pm$ .0093 \\*
 & .75 & .4958 $\pm$ .0159 & .4490 $\pm$ .0138 & \textbf{.4463 $\pm$ .0116} & .4814 $\pm$ .0098 \\*
 & .80 & .5109 $\pm$ .0156 & .4731 $\pm$ .0126 & \textbf{.4718 $\pm$ .0114} & .4989 $\pm$ .0096 \\*
 & .85 & .5229 $\pm$ .0175 & .4993 $\pm$ .0143 & \textbf{.4976 $\pm$ .0120} & .5164 $\pm$ .0091 \\*
 & .90 & .5338 $\pm$ .0168 & .5216 $\pm$ .0161 & \textbf{.5190 $\pm$ .0122} & .5338 $\pm$ .0104 \\*
 & .95 & .5482 $\pm$ .0186 & .5423 $\pm$ .0169 & \textbf{.5400 $\pm$ .0130} & .5496 $\pm$ .0125 \\*
\cline{1-6}
{\tt refrigeration} & .70 & .5739 $\pm$ .0763 & .5203 $\pm$ .0779 & \textbf{.5138 $\pm$ .0785} & .6479 $\pm$ .0603 \\*
 & .75 & .5896 $\pm$ .0732 & .5462 $\pm$ .0845 & \textbf{.5431 $\pm$ .0759} & .6560 $\pm$ .0590 \\*
 & .80 & .6071 $\pm$ .0840 & .5767 $\pm$ .0778 & \textbf{.5694 $\pm$ .0766} & .6665 $\pm$ .0579 \\*
 & .85 & .6181 $\pm$ .0829 & .6087 $\pm$ .0766 & \textbf{.6083 $\pm$ .0713} & .6762 $\pm$ .0556 \\*
 & .90 & \textbf{.6319 $\pm$ .0746} & .6335 $\pm$ .0681 & .6345 $\pm$ .0642 & .6874 $\pm$ .0544 \\*
 & .95 & \textbf{.6607 $\pm$ .0778} & .6669 $\pm$ .0655 & .6711 $\pm$ .0644 & .6997 $\pm$ .0570 \\*
\cline{1-6}
{\tt sonyaiborobot} & .70 & .4444 $\pm$ .0358 & .3373 $\pm$ .0404 & \textbf{.2794 $\pm$ .0381} & .5744 $\pm$ .0463 \\*
 & .75 & .4633 $\pm$ .0405 & .3675 $\pm$ .0432 & \textbf{.3344 $\pm$ .0371} & .5674 $\pm$ .0445 \\*
 & .80 & .4905 $\pm$ .0381 & .4141 $\pm$ .0324 & \textbf{.3847 $\pm$ .0457} & .5607 $\pm$ .0420 \\*
 & .85 & .5078 $\pm$ .0345 & .4470 $\pm$ .0316 & \textbf{.4305 $\pm$ .0435} & .5707 $\pm$ .0390 \\*
 & .90 & .5329 $\pm$ .0368 & .4839 $\pm$ .0317 & \textbf{.4736 $\pm$ .0343} & .5748 $\pm$ .0366 \\*
 & .95 & .5590 $\pm$ .0331 & \textbf{.5248 $\pm$ .0334} & .5277 $\pm$ .0325 & .5737 $\pm$ .0344 \\*
\cline{1-6}
{\tt temperature} & .70 & .6102 $\pm$ .1080 & .4934 $\pm$ .0958 & \textbf{.4919 $\pm$ .0950} & .6720 $\pm$ .1296 \\*
 & .75 & .6666 $\pm$ .1562 & .5364 $\pm$ .0920 & \textbf{.5363 $\pm$ .0922} & .7062 $\pm$ .1364 \\*
 & .80 & .6923 $\pm$ .1644 & \textbf{.5725 $\pm$ .1011} & .5728 $\pm$ .1011 & .7421 $\pm$ .1454 \\*
 & .85 & .7259 $\pm$ .1787 & \textbf{.6142 $\pm$ .1010} & .6146 $\pm$ .1001 & .7776 $\pm$ .1535 \\*
 & .90 & .7757 $\pm$ .1708 & .6775 $\pm$ .1253 & \textbf{.6771 $\pm$ .1239} & .8048 $\pm$ .1572 \\*
 & .95 & .8270 $\pm$ .1648 & .7523 $\pm$ .1433 & \textbf{.7509 $\pm$ .1452} & .8394 $\pm$ .1650 \\*
\bottomrule
\end{longtable}
\begin{figure*}[!t]
  \centering
  \includegraphics[width=1\textwidth]{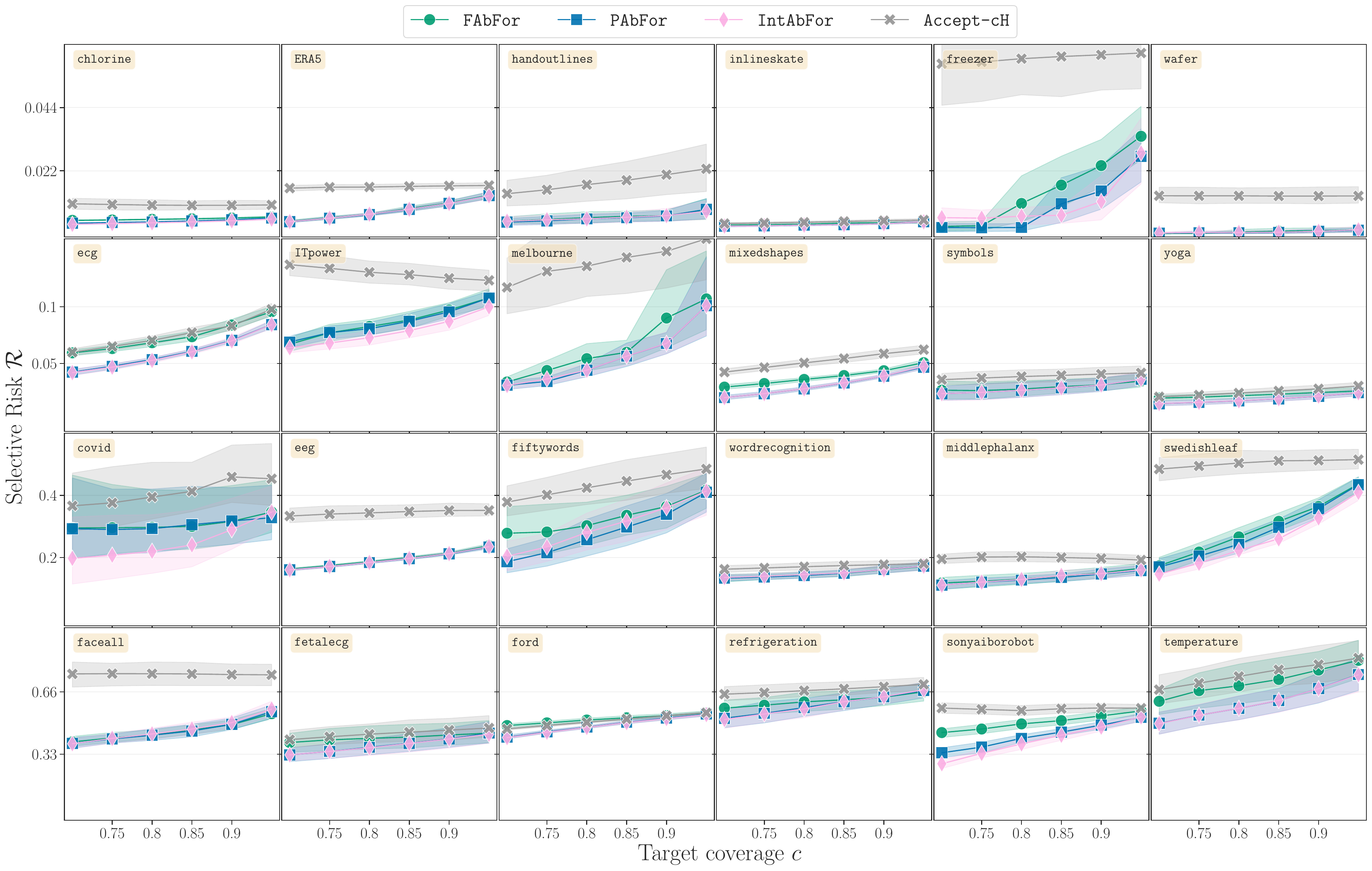}
  \caption{
  Each dataset's average selective risk per step for \firstmethod, \secondmethod, \thirdmethod, and \dummy across six target coverages. 
  }
  \label{fig:Q2app}
\end{figure*}